\documentclass{article}

\usepackage{caption}
\usepackage{subcaption}

\usepackage{microtype}
\usepackage{graphicx}
\usepackage{booktabs} 
\usepackage{pifont}

\usepackage{hyperref}
\usepackage{svg}

\usepackage[accepted]{icml2025}

\usepackage{amsmath}
\usepackage{amssymb}
\usepackage{mathtools}
\usepackage{amsthm}
\usepackage[all]{xy}
\usepackage[capitalize,noabbrev]{cleveref}

\theoremstyle{plain}
\newtheorem{theorem}{Theorem}[section]

\newtheorem{corollary}[theorem]{Corollary}
\theoremstyle{definition}
\newtheorem{definition}[theorem]{Definition}

\theoremstyle{remark}

\theoremstyle{fact}

\usepackage[textsize=tiny]{todonotes}

\icmltitlerunning{Ergodic Generative Flows}

\begin{document}
\def\Rhat{\widehat{R}}
\def\ForwardPolicy{\pi_{\rightarrow}}
\def\ForwardPolicyStar{\pi_{\rightarrow}^*}
\def\BackwardPolicy{\pi_{\leftarrow}}
\def\BackwardPolicyStar{\pi_{\leftarrow}^*}
\def\TV{\mathrm{TV}}
\def\Statespace{\mathcal S}
\def\foutstar{f^*_{\rightarrow}}
\def\Foutstar{F^*_{\rightarrow}}
\def\Fout{F_{\rightarrow}}
\def\finstar{f^*_{\leftarrow}}
\def\Finstar{F^*_{\leftarrow}}
\def\Finit{F_{\mathrm{init}}}
\def\Fterm{F_{\mathrm{term}}}
\def\Finithat{\widehat {F}_{\mathrm{init}}}
\def\Ftermhat{\widehat{F}_{\mathrm{term}}}
\def\ftermhat{\widehat{f}_{\mathrm{term}}}
\def\fterm{f_{\mathrm{term}}}
\def\finit{f_{\mathrm{init}}}
\def\LFMstable{ \mathcal L_{\mathrm{FM},2}^{\mathrm{stable}} }
\def\LFMstableq{ \mathcal L_{\mathrm{FM},q}^{\mathrm{stable}} }

\def\KL{\mathrm{KL}}
\def\deltafinit{\delta f_{\mathrm{init}}}
\def\LKLwFM{\mathcal L_{KL-wFM}}
\def\target{\kappa}
\def\source{\nu}
\def\backwardalpha{\alpha_{\leftarrow} }
\def\forwardalpha{\alpha_{\rightarrow} }

\twocolumn[
\icmltitle{Ergodic Generative Flows}



\icmlsetsymbol{equal}{*}

\begin{icmlauthorlist}
\icmlauthor{Leo Maxime Brunswic}{equal,mtl}
\icmlauthor{Mateo Clemente}{equal,mtl}
\icmlauthor{Rui Heng Yang}{mtl}
\icmlauthor{Adam Sigal}{mtl}
\icmlauthor{Amir Rasouli}{mtl}
\icmlauthor{Yinchuan Li}{huawei}
\end{icmlauthorlist}

\icmlaffiliation{mtl}{Huawei Technologies Canada, Noah's Ark Laboratories}
\icmlaffiliation{huawei}{Huawei}

\icmlcorrespondingauthor{Leo Maxime Brunswic}{leo.maxime.brunswic@h-partners.com}

\icmlkeywords{Machine Learning, ICML}

\vskip 0.3in
]



\printAffiliationsAndNotice{\icmlEqualContribution} 

\begin{abstract}
Generative Flow Networks (GFNs) were initially introduced on directed acyclic graphs to sample from an unnormalized distribution density. Recent works have extended the theoretical framework for generative methods allowing more flexibility and enhancing application range. However, many challenges remain in training GFNs in continuous settings and for imitation learning (IL), including intractability of flow-matching loss, limited tests of non-acyclic training, and the need for a separate reward model in imitation learning. The present work proposes a family of generative flows called Ergodic Generative Flows (EGFs) which are used to address the aforementioned issues. First, we leverage ergodicity to build simple generative flows with finitely many globally defined transformations (diffeomorphisms) with universality guarantees and tractable flow-matching loss (FM loss). Second, we introduce a new loss involving cross-entropy coupled to weak flow-matching control, coined KL-weakFM loss. It is designed for IL training without a separate reward model. We evaluate IL-EGFs on toy 2D tasks and real-world datasets from NASA on the sphere, using the KL-weakFM loss.  Additionally, we conduct toy 2D reinforcement learning experiments with a target reward using the FM loss.
\end{abstract}

\section{Introduction}
\label{Introduction}
    Generative models aim to sample from a target distribution $\target$ on a space $\Statespace$, possibly conditioned on additional context or variables. The target distribution $\target$ is either specified by an unnormalized density with respect to a background measure $\lambda$, as in reinforcement learning (RL), or defined by a dataset of samples $\mathcal D$ as in imitation learning (IL). While RL \cite{brooks2011handbook, sutton2018reinforcement, haarnoja2018soft, bengio2023gflownet} and IL \cite{ goodfellow2020generative, ho2020denoising, song2020score,rozen2021moser, papamakarios2021normalizing} methods are often considered to be separate approaches, it is a common practice to first train a reward model from a dataset and then use RL techniques to solve an IL task \cite{chung2024scaling, zhang2022generative,zhang2022unifying,sendera2024improved}.
    
    Among IL methods, normalizing flows (NFs), although capable of leveraging advanced neural ODEs \cite{chen2018neural,grathwohl2018ffjord}, have not achieved state-of-the-art performance due to architectural constraints (i.e. the requirement of sequentially combined parameterized diffeomorphisms) and numerical instabilities \cite{caterini2021variational,verine2023expressivity}. 

    While NFs are deterministic models, stochasticity has proven to be a crucial component of IL generative models, such as diffusion models as its simplest embodiment \cite{ho2020denoising,song2020score}. However, diffusion models face a challenging trade-off between loss function tractability and the time-consuming denoising process for generation. This trade-off drives continued efforts to accelerate the diffusion inference process \cite{lu2022dpm, tang2025adadiff, chen2024accelerating}. Such issues are particularly important when the computational and time budgets of generation are highly constrained, which is typical on real-time tasks or mobile devices. The present work leverages the flexibility of Generative Flow Networks (GFNs) \cite{bengio2021flow, bengio2023gflownet} to address these issues by building simple yet highly expressive models with short sampling trajectories. 

   Generative Flow Networks (GFNs), originally formulated for RL tasks, are a family of generative methods designed to sample proportionally based on a reward function $r = \varphi_\target$, which represents the density of a target distribution $\target$. They were initially restricted to directed acyclic graphs, but have since been extended to other settings. Subsequent work has extended GFNs to generalize to continuous state spaces \cite{li2023cflownets, lahlou2023theory} and non-acyclic structures \cite{brunswic2024theory}.
    
    Despite these advancements, GFNs still face four key challenges that hinder their application in RL and IL settings. First, although it has been suggested that GFNs themselves could be used as a reward model in the IL setting (see \citet{bengio2023gflownet} p37), previous works \cite{zhang2022generative, lahlou2023theory} train a separate reward model on samples from $\target$, effectively reducing the problem to an RL task. The need to train an additional reward model, which results in additional training and computation costs, stems from a lack of control of the so-called flow-matching (FM) property of the flow. Indeed, the FM-loss is zero by construction of the GFN self-defined reward model.
    
    Second, the FM training loss presents significant challenges in continuous settings, as its evaluation is intractable in naive implementations \cite{li2023cflownets}, necessitating the training of an imperfect estimator when directly enforcing the FM constraint. Addressing this issue requires either handcrafting or training an additional backward policy $\BackwardPolicyStar$, or resorting to higher-variance loss functions, such as the Detailed Balance (DB) \cite{bengio2021flow} and Trajectory Balance (TB) \cite{malkin2022trajectory} losses. 
    
    Third, the acyclicity requirements \cite{lahlou2023theory} mean that additional structures must be manually crafted. Cycles appear naturally in naive implementations and RL environments. \citet{brunswic2024theory} offers a theory for non-acyclic flows and stable FM losses, but these have yet to be tested in strongly non-acyclic settings. 
    
    Lastly, \citet{li2023cflownets} argues that exact cycles are negligible in the continuous setting due to their zero probability, whereas \citet{brunswic2024theory} argues that 0-flows, non-zero measures $\xi$ invariant under $\ForwardPolicyStar$ ($\xi\ForwardPolicyStar=\xi$), need to be addressed. Otherwise, the divergence-based losses used by \citet{bengio2023gflownet} may become unstable in the presence of $\xi$ — a prediction that remains untested.
    
    We present Ergodic Generative Flows (EGFs), which address the aforementioned key limitations. EGFs are built using a policy choosing at a given state from finitely many globally defined transformations, i.e. diffeomorphisms of the state space. This allows for tractable inverse flow policy, hence tractable FM-loss. Since we favor finitely many simple transformations, the universality of EGFs is non trivial. We provide provable guarantee with an {\it ergodicity} \cite{walters2000introduction, kifer2012ergodic} assumption of the group generated by the EGFs transformations. 
    
    \textbf{Main contributions}: 
    \begin{enumerate}
      \item We extend the theoretical framework of generative flows, in particular providing quantitative versions of the sampling theorem for non-acyclic generative flows.
      \item  We develop a theory of EGFs, presenting a universality criterion and demonstrating that, in any dimension on tori and spheres, this criterion can be fulfilled with four simple transformations.
      \item  We propose a coupled KL-weakFM loss to train EGFs directly for IL. This allows to train EGFs for IL without a separate reward model.
    \end{enumerate}

\section{Preliminaries: GFN for RL and IL}
\label{sec:preliminaries}
    In this work, we focus on generative flows in the continuous setting, encompassing scenarios where the state space $\Statespace$ is either a vector space, $\Statespace = \mathbb{R}^d$, or potentially a Riemannian manifold \cite{lee2018introduction, gemici2016normalizing}. Both are instances of Polish\footnote{A topological measurable space is Polish if its topology is completely metrizable and it is endowed with its Borel $\sigma$-algebra.} spaces that admit a natural finite background measure $\lambda$ (usually the Lebesgue, Haar \cite{halmos2013measureHaar} or Riemannian volume measures). 
    
    A measure $\mu$ is dominated by $\nu$, denoted $\mu\ll \nu$ if $\nu(A)=0 \Rightarrow \mu(A)=0$ for any measurable $A\subset \Statespace$; then, the Radon-Nikodym derivative of $\mu$ with respect to $\nu$, denoted $\frac{d\mu}{d\nu}\in L^1(\nu)$, is the unique $\nu$-integrable function $\varphi$ such that $\varphi \nu = \mu$. A 
    Markov kernel $\pi:\mathcal X\rightarrow \mathcal Y$ is a stochastic map denoted $\pi(x)$. The image measure of $\mu$ by $\pi$ denoted $\mu \pi$ is defined by $\mu\pi(A) = \int_{x\in \mathcal X} \mathbb P(\pi(x)\in A) d\mu(x)$ for any measurable $A\subset \mathcal Y$. Their tensor product is the kernel  $\mu\otimes \pi(A\rightarrow B):=\int_{x\in A}\mathbb P(\pi(x)\in B)d\mu(x)$.
    Even when $\mu$ is an unnormalized distribution, we write $x\sim \mu$ for ``the law of $x$ is $\frac{1}{\mu(\Statespace)}\mu$".
    
    We adopt a formulation of generative flows that differs from \citet{lahlou2023theory} and \citet{brunswic2024theory}, providing the necessary flexibility for the subsequent sections. Let $(\Statespace,\lambda)$ be a Polish space with a finite background measure. A generative flow on $\Statespace$ is the data of a star forward policy, i.e. a Markov kernel $\ForwardPolicyStar:\Statespace\rightarrow \Statespace$, and a star outflow, i.e. a finite non-negative measure $\Foutstar$ on $\Statespace$. 
    
    With unnormalized initial and terminal distributions $\Finit$ and $\Fterm$,  it induces a Markov chain $(\underline s_t)_{t\geq 1}$ defined by $\underline s_{1}\sim \Finit$ and $\underline s_{t+1} = \ForwardPolicyStar(\underline s_t)$. 
    Define the sampling time $\tau \in \mathbb N_{\geq 1}\cup\{\infty\}$ as the random variable given by\footnote{Clearly $\Fterm \ll \Fterm+\Foutstar$ so the Radon-Nikodym derivative is well-defined, up to a $(\Fterm+\Foutstar)$-null set; under the flow-matching constraint \eqref{equ:FM_const} the chain's unstopped mass $\sum_{t\geq 1}\mathbb P(\tau\geq t,\underline s_t\in\cdot)$ is dominated by $(\Fterm+\Foutstar)/\Finit(\Statespace)$, so it charges no $(\Fterm+\Foutstar)$-null set, and the sampler obeys the hazard above for every version.} $\mathbb P(\tau\geq 1)=1$ and $\mathbb P(\tau = t|\tau\geq t) = \frac{d\Fterm}{d(\Fterm+\Foutstar)}(\underline s_t)$, that is, at each $t$ with $\tau\geq t$ the sampler stops with this probability, by a coin independent of the chain and of the earlier coins, the chain moving on by $\ForwardPolicyStar$ regardless (a randomized stopping time).
    The sampler associated with a generative flow is obtained by emulating the Markov chain $(\underline s_t)_{t\geq1}$ and sampling $\underline s_\tau$.
    By extending the sampling Markov chain with a source $s_0$ and sink $s_f$ state, we can summarize the flow as follows:
    \begin{equation}
        \xymatrix{
        s_0 \ar[r]^\Finit&\Statespace \ar@(ur,ul)_{\Foutstar\otimes\ForwardPolicyStar}\ar[r]^\Fterm& s_f
        }.
    \end{equation}

  The following theorem  provides theoretical guarantees on the sampler associated with a generative flow.

    \begin{theorem}
        [\cite{bengio2021flow, brunswic2024theory}] \label{theo:sampling_theorem} 
        Let $\Finit$ and $\Fterm$ be unnormalized distributions with $\Fterm\neq0$ and let $(\ForwardPolicyStar, \Foutstar)$ be a generative flow. If the flow-matching constraint 
        \begin{equation}
        \label{equ:FM_const}
            \Finit + \Foutstar\ForwardPolicyStar = \Fterm + \Foutstar
        \end{equation}
        is satisfied then $\Finit(\Statespace)=\Fterm(\Statespace)$, $\tau<\infty$ almost surely and:
        \begin{equation}
            \mathbb E(\tau)\leq \frac{\Foutstar(\Statespace)}{\Finit(\Statespace)}+1,\quad \quad \underline s_\tau \sim \Fterm. \nonumber
        \end{equation}
    
    \end{theorem}

    The pair $(\ForwardPolicyStar, \Foutstar)$ is trained using gradient descent to minimize a loss function that enforces the flow-matching constraint in equation \eqref{equ:FM_const}. Once convergence is achieved, the generative flow sampler yields samples of the random variable $s_\tau$. Hence, by Theorem \ref{theo:sampling_theorem},  the generative flow sampler samples from $\Fterm$. For example, we can use the most straightforward stable FM-loss:
    \begin{equation}\label{equ:FMloss}
        \LFMstableq = \mathbb E_{s\sim \nu_{\mathrm{train}}}\left[ \left(\finit + \finstar - \fterm - \foutstar\right)^q(s) \right]
    \end{equation}
    where $\finit = \frac{d\Finit}{d\lambda}$, $\fterm=\frac{d\Fterm}{d\lambda}$ and $\finstar=\frac{d\Finstar}{d\lambda} = \frac{d(\Foutstar\ForwardPolicyStar)}{d\lambda}$; the latter presupposes $\Finstar\ll\lambda$, which holds for an EGF since its moves are nonsingular (Section \ref{sec:fundamentals_EGF}). 
    
    In general, both $\finit$ and $\foutstar$ are determined by the chosen parametrization, with $\foutstar$ being problem-dependent, while $\finstar$ requires the computation of an image measure density. However, the previously mentioned methods outlined above face several limitations:

    \textbf{Limitation \ding{172}:} The first limitation arises in IL. $\fterm$ is unknown so one has to build a density model for $\Fterm$. Disregarding solutions based on separate energy models, we can attempt to set $\Fterm= \Finit + \Finstar - \Foutstar$, assuming $\Finstar$ is tractable.
    We then use a cross-entropy loss to train $\Fterm$ to approximate $\target$. Despite the FM loss formally and trivially becoming 0, $\Fterm$ may not be positive. Therefore, the existing theoretical framework fails and so does the sampling. At this stage, we move beyond the existing theoretical framework.

    \textbf{Limitation \ding{173}:} The second limitation arises in RL. At the inference time, the reward may be unknown, either because of its costly nature or because it is simply inaccessible. A straightforward solution consists of using $\Fterm = \target$ during training and $\Fterm= \Finit + \Finstar - \Foutstar$ during inference. However, this reward-less sampling is ill-controlled because $\Fterm$ may not be positive.

    \textbf{Limitation \ding{174}:} Both training and inference are based on the tractability of $\finstar$. For instance, naive forward policies, such as those used in diffusion models add noise in $\mathbb{R}^d$. These policies are typically defined as $\ForwardPolicyStar(x) = m(x) + \epsilon$, where $m$ is a deterministic model and $\epsilon \sim \mathcal{N}(0, \eta)$ represents Gaussian noise. This leads to an intractable star inflow as the integral in equation \eqref{equ:intractable_inflow} becomes computationally intensive.
    \begin{equation}
        \label{equ:intractable_inflow}
         \finstar(y)  \propto \int_{x\in \mathbb R^d} \foutstar(x) \exp\left(-\frac{(m(x)-y)^2}{2\eta^2}\right)d\lambda(x).
    \end{equation}

    \textbf{Limitation \ding{175}:} Although \citet{li2023cflownets} argues that exact cycles are negligible in the continuous setting due to their zero probability, \citet{brunswic2024theory} disputes that their generalization called 0-flows must still be addressed. 0-flows are the non-zero measures $\xi$ invariant under $\ForwardPolicyStar$, $\xi\ForwardPolicyStar=\xi$, that is the outflows of the generative flows $(\ForwardPolicyStar,\xi)$ satisfying the flow-matching constraint~\eqref{equ:FM_const} with $\Finit=\Fterm=0$. \citet{brunswic2024theory} predicts that divergence-based losses used by \citet{bengio2023gflownet} are unstable if such a $\xi$ exists. This has yet to be tested.
        
\section{Ergodic Generative Flows for  RL and IL} \label{sec:fundamentals_EGF}
    To address the four issues outlined above, we propose Ergodic Generative Flows (EGFs).
   
    \begin{definition}[Ergodic Generative Flows]\label{def:egf}
        Let $\Statespace$ be a Riemannian manifold endowed with its volume measure $\lambda$ and let $(\Phi_{i})_{i=1}^p$ be a family of diffeomorphisms $\Statespace\rightarrow \Statespace$. An EGF is a generative flow, where  
        \begin{eqnarray} \label{equ:forward_policy}
            \ForwardPolicyStar(s) = \sum_{i=1}^p\forwardalpha^i(s) \delta_{\Phi_i(s)}
        \end{eqnarray}    
        for some measurable policy $\forwardalpha:\Statespace \rightarrow [0,1]^p$ with $\sum_i\forwardalpha^i=1$, and such that the group of diffeomorphisms generated by the $\Phi_i$ is ergodic for $\lambda$. That is, every measurable $A\subset\Statespace$ with $\lambda\left(\Phi_i^{-1}(A)\,\triangle\,A\right)=0$ for all $i$ satisfies $\lambda(A)=0$ or $\lambda(\Statespace\setminus A)=0$.
    \end{definition}
    A Markov kernel $\pi$ is \emph{ergodic} for a measure $\mu$ if every measurable $A\subset\Statespace$ with $\pi(x\rightarrow A)=\mathbf 1_A(x)$ for $\mu$-almost every $x$ satisfies $\mu(A)=0$ or $\mu(\Statespace\setminus A)=0$; this depends on $\mu$ only through its null sets. If the policy $\ForwardPolicyStar$ of \eqref{equ:forward_policy} is ergodic for $\lambda$, the group of Definition~\ref{def:egf} is ergodic for $\lambda$, and conversely if every $\forwardalpha^i>0$. Unless a measure is named, an ergodic policy is one ergodic for a finite invariant measure equivalent to $\lambda$, as in Theorems~\ref{theo:universality} and~\ref{theo:universality_true}.

    There are two key ideas that motivate Definition~\ref{def:egf}:

     \textbf{Tractability} of  star inflow $\finstar$ is achieved by only using finitely many diffeomorphisms. Let us say that $\Statespace=\mathbb R^d$ endowed with Lebesgue measure $\lambda$ and that each move $\Phi_i$ is a parametrizable diffeomorphism. Then, a closed form formula for the backward policy may be deduced from a detailed balance argument, see appendix \ref{sec:appendix_backward} for details:
    \begin{eqnarray}\label{equ:tractable_inflow}
\BackwardPolicy^*(s) &=& \sum_{i=1}^p\backwardalpha^i(s) \delta_{\Phi_i^{-1}(s)} \\ 
\backwardalpha^i(s) &=& \frac{(\forwardalpha^if^*_\rightarrow)\circ \Phi_i^{-1}(s)|\det J_s {\Phi_i^{-1}}| }{ \sum_j (\forwardalpha^jf^*_\rightarrow)\circ \Phi_j^{-1}(s) |\det J_s {\Phi_j^{-1}}|}
\end{eqnarray}
where $J_s$ denotes the Jacobian matrix at $s$.
Denote by $(\Phi_i)_\#\mu := \mu(\Phi_i^{-1}(\cdot))$ the image of a measure $\mu$ by $\Phi_i$. For an EGF, $\Finstar = \sum_i (\Phi_i)_\#(\forwardalpha^i \Foutstar) \ll \lambda$ because $\Foutstar \ll \lambda$ and each $\Phi_i$ is nonsingular, $(\Phi_i)_\#\lambda \ll \lambda$ (true for diffeomorphisms and the Riemannian volume); hence $\finstar$ exists ($\lambda$ being $\sigma$-finite).
Therefore, we have a closed-form formula for the density of the star inflow,
\begin{equation} \label{equ:finstar}
    \finstar(s)  = \sum_{i=1}^p  (\forwardalpha^i\foutstar)\circ \Phi_i^{-1}(s) |\det J_s \Phi_i^{-1}|.
\end{equation}
All terms in equation \eqref{equ:finstar} are tractable, so  $\finstar$ is tractable as long as the number of diffeomorphisms $p$ is small. 
As a consequence, the flow-matching loss given in Equation \eqref{equ:FMloss} is tractable for $p$ small enough.

\textbf{Expressivity} can potentially be an issue if the number of diffeomorphisms is small and each move is too simple \cite{rezende2020normalizing}. Ergodicity guarantees, for a model rich enough (Theorems~\ref{theo:universality}, \ref{theo:universality_tori} and~\ref{theo:universality_sphere}), that the parameterized family is universal: for any probability distributions $\Finit,\Fterm\ll \lambda$ with bounded density it contains flows from $\Finit$ to $\Fterm$ with arbitrarily small errors $\delta\Finit(\Statespace)+\delta\Fterm(\Statespace)$, even with simple transformations, such as affine maps on tori and rotations on spheres.


The following sub-sections are organized as follows: Section \ref{sec:universality} establishes the universality of EGFs. In Section \ref{sec:theory1}, we present theoretical advancements for general generative flows, which form the basis for the design of the KL-weakFM loss used for IL training, described in Section \ref{sec:ILloss}.

\subsection{Universality From Ergodicity}
\label{sec:universality}
We begin with a formalization of universality in conjunction with the master universality Theorem.
\begin{definition}
On a state space $(\Statespace,\lambda)$, a family of generative flows is 
universal if for any two probability distributions $\mu,\nu\ll \lambda$ with bounded density and every $\varepsilon>0$ there is a flow $(\ForwardPolicyStar,\foutstar)$ in the family from $\Finit=\nu$ to $\Fterm=\mu$ with $\delta\Finit(\Statespace)+\delta\Fterm(\Statespace)<\varepsilon$, where $\delta\Finit,\delta\Fterm$ are the initial and terminal errors defined in Section \ref{sec:theory1}.
\end{definition}

We begin with the presentation of our master universality theorem in section \ref{sec:universality_master}. This theorem is then applied to tori and spheres, providing simple examples of universal families, as detailed in sections \ref{sec:universality_tori} and \ref{sec:universality_sphere}.

\subsubsection{Master Universality Theorem}
\label{sec:universality_master}
The next theorem states that if the family contains an ergodic policy $\ForwardPolicyStar$, then the family is universal. We first recall a stronger property, which Theorem~\ref{theo:universality} does not need.
\begin{definition}\label{def:summably_mixing} A Markov kernel $\pi$ on $(\Statespace,\lambda)$ is summably $L^2$-mixing if $\pi$ admits an invariant measure $\widehat \lambda$ equivalent to $\lambda$ and such that the $L^2$-mixing coefficients given by,
\begin{equation}
    \gamma_n := \sup_{\varphi \in H} \int_{s\in \Statespace}
    \left(\varphi\pi^n(s)\right)^2 d\widehat \lambda(s), 
\end{equation}
with $\varphi\pi^n := \frac{d((\varphi\lambda)\pi^n)}{d\lambda}$ and $H = \{\varphi \in L^2(\lambda)~|~\|\varphi\|_{L^2(\lambda)} = 1 \text{ and } \int_{\Statespace} \varphi\, d\lambda=0\}$
are summable: $\sum_{n\geq 0} \gamma_n < +\infty$.
\end{definition}

The technical summably $L^2$-mixing property informally means that iterating the Markov kernel $\pi$ is averaging out functions fast enough to ensure that the sum of errors is finite. Ergodicity alone ensures the convergence of $\varphi\pi^n$ only in the much weaker Ces\`aro sense, and this is enough for universality (Theorem~\ref{theo:universality}, Appendix~B.1).

This property is in particular satisfied if $\lambda$ is invariant, $\lambda\pi=\lambda$, and the Koopman operator $\phi \mapsto \frac{d (\phi\lambda)\pi}{d\lambda}$ has spectral radius $<1$ on the subspace $\{\phi~|~\int \phi\, d\lambda =0\}$ (for a reversible $\pi$: its spectrum there lies in $[-1+c,1-c]$ for some $c>0$); a spectrum bounded away from $1$ only is not enough (the swap of two points). In this case, the convergence of iterates by $\pi$ is exponential.

\begin{theorem} \label{theo:universality}A parameterized family  of EGF is universal provided  that it contains some $\ForwardPolicyStar$ ergodic for a finite invariant measure equivalent to $\lambda$  and that the set of $\foutstar$ such that $(\ForwardPolicyStar,\foutstar)$ is in the family is dense in the non-negative cone of $L^2(\Statespace,\lambda)$.
\end{theorem}

\subsubsection{Ergodicity on Tori}
\label{sec:universality_tori}
For the sake of simplicity, let $\mathbb T^d$ be a flat torus of side 1 endowed with Lebesgue measure $\lambda$ and let $P:\mathbb R^d \rightarrow \mathbb T^d$ be the natural universal covering projection. The simplest family of transformations is the group of affine transformations,
\begin{equation}\mathrm{Aff}(\mathbb T^d) :=
\{x\mapsto Ax+b ~|~ b\in \mathbb R^d \text{ and } A\in \mathrm{SL}_d(\mathbb Z)\},
\end{equation}
with $ \mathrm{SL}_d(\mathbb Z)$ denoting the set of square matrices of order $d$ with integral coefficients and determinant 1. 

Given that $(\Phi_i)_{i=1}^p$ are all in $\mathrm{Aff}(\mathbb T^d)$, we can build a continuous family of generative flows using equation \eqref{equ:forward_policy}, which is then parameterized by the translation part of each $\Phi_i$ together with  a model having a softmax head $\forwardalpha$ and a scalar head $\foutstar$. We call such a family an {\it affine toroidal} family.

\begin{theorem}\label{theo:universality_tori} An affine toroidal family is universal provided that, for some translation parts of the $(\Phi_i)_{i=1}^p$, its model realizes the uniform policy $\forwardalpha^i\equiv\frac1p$ together with a set of outflows dense in the non-negative cone of $L^2(\Statespace,\lambda)$ and the group generated by these $\Phi_i$ acts ergodically on $(\mathbb T^d,\lambda)$ (the flows of the family at these translation parts are then EGFs); this holds when that group has no non-trivial invariant factor torus on which its linear part projects to an amenable group (Theorem 5 of \citet{bekka2015spectral}), and when it contains an affine map whose linear part has no periodic integer vector (e.g. $T(STS^{-1})^{-1}=\begin{pmatrix}2&1\\1&1\end{pmatrix}$ for the minimal four moves, on $\mathbb T^2$ with linear parts $S^{\pm1},T^{\pm1}$, where $S=\begin{pmatrix}0&-1\\1&0\end{pmatrix}$ and $T=\begin{pmatrix}1&1\\0&1\end{pmatrix}$) or, on $\mathbb T^1$, a translation of infinite order (Appendix~B.2).
\end{theorem}
The ergodicity assumption above is in particular satisfied, for $d\geq2$, if the group generated by the linear parts of the $\Phi_i$ is the whole group $\mathrm{SL}_d(\mathbb Z)$ (Appendix~B.2; for $d=2$ the group then contains an affine map with linear part $\begin{pmatrix}2&1\\1&1\end{pmatrix}$, which has no periodic integer vector), and for $d=1$ if one translation part is irrational. It is shown in \citet{conder2025generating} that $\mathrm{SL}_d(\mathbb Z)$ has a presentation with two generators. Therefore, in any dimension, EGF are universal with $p=4$ and well-chosen $(\Phi_i)_{i=1}^4$: the two generators and their inverses.

\subsubsection{Ergodicity on Spheres}
\label{sec:universality_sphere}
Consider the round sphere $\mathbb S^d=\{x\in \mathbb R^{d+1}~|~ \|x\|=1\}$ endowed with its natural Riemannian volume measure $\lambda$. Again a simple family of diffeomorphisms is provided by the so-called projective transformations, 
\begin{equation}
    \mathrm{PGL}_{d+1}(\mathbb R) := \{x\mapsto Ax/\|Ax\| ~|~ A\in \mathrm{GL}_{d+1}(\mathbb R)\},
\end{equation}
where $\mathrm{GL}_{d+1}(\mathbb{R})$ is the group of invertible matrices of order $d+1$.
Similar to tori, we build a continuous family of generative flows using equation~\eqref{equ:forward_policy} by parametrizing $(\Phi_i)_{i=1}^p$ in $\mathrm{PGL}_{d+1}(\mathbb R)$ which is a Lie group, and choose a tractable model that have a softmax head $\forwardalpha$ and a scalar head $\foutstar$. Such a family is called a {\it projective spherical} family. 
We focus on the sub-family of {\it isometry spherical} family composed of rotations $\Phi_i\in \mathrm{SO}_{d+1}(\mathbb R)$ for which we have theoretical guarantees.

\begin{theorem}\label{theo:universality_sphere} An isometry spherical family is universal provided that, for some rotations $(\Phi_i)_{i=1}^p$ of the family, its model realizes $\forwardalpha^i\equiv\frac1p$ together with a set of outflows dense in the non-negative cone of $L^2(\Statespace,\lambda)$ and these $\Phi_i$ generate a dense subgroup of $\mathrm{SO}_{d+1}(\mathbb R)$, $d\geq1$ (the flows of the family at these rotations are then EGFs).
\end{theorem}
The density assumption holds at every parameter if the $\Phi_i$ are frozen and generate a dense subgroup of $\mathrm{SO}_{d+1}(\mathbb R)$, $d\geq1$: the group they generate then acts ergodically on $\mathbb S^d$ (Appendix~B.3). Neither a spectral gap \cite{bourgain2012spectral,benoist2016spectral} nor algebraic coefficients are needed. Furthermore, the main Theorem of \citet{breuillard2003dense} allows to build a dense subgroup of $\mathrm{SO}_{d+1}(\mathbb R)$ with two generators and algebraic coefficients.
Therefore, EGF are universal on spheres with $p=4$ and well-chosen $(\Phi_i)_{i=1}^4$: the two generators and their inverses.

\subsection{Quantitative Sampling Theorem}
\label{sec:theory1}

We present a quantitative sampling theorem for generative flows. The key idea of the proof is that any generative flow matches the flow up to a transformation of the initial and terminal distributions.

\begin{definition}[Virtual initial and terminal flow] Let $(\ForwardPolicyStar,\Foutstar)$ be a generative flow and let $\Finit,\Fterm$ be initial and terminal distributions. Define the initial and terminal errors as $\delta\Finit:=(\Finit+\Finstar-\Fterm-\Foutstar)^-$ and $\delta\Fterm:=(\Finit+\Finstar-\Fterm-\Foutstar)^+$ respectively. Define the positive and negative parts of a measure $\mu$ as ${\mu}^\pm$.
The virtual initial and terminal distributions are $\Finithat := \Finit+ \delta\Finit$ and $\Ftermhat := \Fterm+ \delta\Fterm$.
\end{definition}

We note that once the initial and terminal distributions $\Finit,\Fterm$ are specified, then any generative flow $(\ForwardPolicyStar,\Foutstar)$ is a FM with respect to the virtual initial and terminal distributions, $\Finithat$ and $\Ftermhat$. This enables the presentation of a quantitative formulation of the sampling theorem for generative flows. 
\begin{theorem}[Quantitative Sampling of Generative Flows]
\label{theo:negative_control}
    Let $(\ForwardPolicyStar, \Foutstar)$ be a generative flow and let $\Finit$ and $\Fterm$ initial and terminal distributions. Assume that $\Finit\neq 0$ and consider the sample $s_\tau$ of the generative flow Markov chain from $\Finit$ to $\Ftermhat$ then:
    \begin{equation}
        \TV\left(s_\tau \left\| \frac{1}{\Ftermhat(\Statespace )} \Ftermhat\right.\right) \leq  \frac{\delta\Finit(\Statespace)}{\Ftermhat(\Statespace)},
    \end{equation}
    with $\TV(\mu\|\nu) := \frac12\|\mu-\nu\| = \sup_{A}|\mu(A)-\nu(A)|$ the total variation distance between probability measures, where $\|\mu-\nu\|:=(\mu-\nu)^+(\Statespace)+(\mu-\nu)^-(\Statespace)$ and $A$ ranges over the measurable subsets of $\Statespace$ (the bound is attained on a two-point example).
\end{theorem}
\begin{corollary} For a generative flow trained on the target $\target$ with $\Finit(\Statespace)=1$ and using $\Fterm=\widehat\target := \left(\Finit+\Finstar-\Foutstar\right)^+$ during inference, the sampling error is controlled by:
    \begin{eqnarray}
      \TV\left(s_\tau \left\| \frac{\target}{\target(\Statespace )} \right.\right) \leq  \frac{\delta}{1+\delta}+\TV\left(\frac{\widehat\target}{\widehat\target(\Statespace)} \left\| \frac{\target}{\target(\Statespace )} \right.\right),
\end{eqnarray}
with $\delta :=  \left(\Finit+\Finstar-\Foutstar\right)^-(\Statespace)$.
\end{corollary}
\begin{proof}
At $\Fterm=\widehat\target$ the terminal error vanishes, so $\Ftermhat=\widehat\target$, $\delta\Finit=(\Finit+\Finstar-\Foutstar)^-$ and $\Ftermhat(\Statespace)=1+\delta$; Theorem \ref{theo:negative_control} gives $\TV(s_\tau\|\widehat\target/\widehat\target(\Statespace))\leq\delta/(1+\delta)$, and the triangle inequality for $\TV$ concludes. The bound holds for every generative flow and every finite $\target\neq0$.
\end{proof}
Commonly the training loss is designed to control the total variation term in the upper bound. However, a secondary regularization term controlling $\delta =\delta\Finit$ may be added to enhance reward-less sampling. 

To substantiate this quantitative bound, let's consider the RL case with a reward $r$ defined on a finite DAG $(\mathcal V,\mathcal E)$ with $\mathcal V=\Statespace\cup\{s_0,s_f\}$, trained using a flow-matching loss, say $\LFMstableq$. Let's take $\nu_{\mathrm{train}}$ the uniform distribution on the vertices and $q=1$, then $\LFMstableq= \frac{1}{|\mathcal V|}\left(\delta + \sum_{s\in\Statespace} |\widehat r(s) - r(s)|\right)$ with $\widehat r(s)$ the density of $\widehat \kappa$ with respect to the counting measure.
We may also rewrite $\TV\left(\frac{\widehat\target}{\widehat\target(\Statespace)} \left\| \frac{\target}{\target(\Statespace )} \right.\right) =\frac{1}{2}\sum_{s\in\Statespace} |\frac{\widehat r(s)}{\sum_{s'\in \Statespace}\widehat r(s')} - \frac{ r(s)}{\sum_{s'\in \Statespace} r(s')}|$ 
so that if $Z=\sum_{s\in \Statespace} r(s)$ then
\begin{equation}
     \TV\left(s_\tau \left\| \frac{\target}{\target(\Statespace )} \right.\right) \leq \left( 1+\frac{1}{Z}\right)|\mathcal V| \LFMstableq.
\end{equation}

    

\subsection{KL-WeakFM loss and IL algorithm}
\label{sec:ILloss}
Henceforth, we assume that $\Statespace$ is a Polish space equipped with a background measure $\lambda$ and $\target$ is a target probability distribution.

Leveraging Theorem \ref{theo:negative_control}, we design the KL-weakFM loss, $\mathcal L_{KL-wFM}$, which enables IL training of a generative flow without requiring a separate reward model.
Let $\deltafinit=\max(0,-(\finit+\finstar-\foutstar))$, the negative part of $\finit+\finstar-\foutstar$,  and $\ftermhat = \max(0,\finit+\finstar-\foutstar)$. The KL-weakFM loss can then be defined as follows:
\begin{multline}\label{equ:weakFM}
     \mathcal L_{KL-wFM} (\theta) :=\\ b\mathbb E_{s\sim \nu_{\mathrm{train}}} \deltafinit( s)-  \mathbb E_{s\sim \target}\log \ftermhat(s),
\end{multline}
for some training probability distribution $\nu_{\mathrm{train}}$ on $\Statespace$ (a law of the states of training paths, as in Algorithm~\ref{alg:IL_EGF}) with $\lambda\ll\nu_{\mathrm{train}}$ and $\|d\lambda/d\nu_{\mathrm{train}}\|_\infty<\infty$, and $b>0$ (Algorithm~\ref{alg:IL_EGF} takes $b\geq1+\|d\lambda/d\nu_{\mathrm{train}}\|_\infty$, for which the loss controls the sampling error).

The name of the loss is derived from two components: on the one hand, the cross-entropy term, which controls the Kullback-Leibler divergence between $\Ftermhat$ and $\target$; on the other hand, the term  $\mathbb E_{s\sim \nu_{\mathrm{train}}} \deltafinit( s)$, which resembles the FM-loss but controls only the negative part of the FM defect.
We now provide a detailed description and motivation for $\LKLwFM$. The discussion begins with a corollary of Theorem \ref{theo:negative_control}, based on Pinsker's inequality:

\begin{corollary} \label{cor:pinsker_negative_control}
Let $(\ForwardPolicyStar, \foutstar)$ be a generative flow. We train the generative flow on target probability distribution $\target$ with $\Finit(\Statespace)=1$. By using $\Fterm=\widehat\target:=  \left(\Finit+\Finstar-\Foutstar\right)^+$ during inference, the sampling error is bounded as follows:
\begin{eqnarray}\label{equ:control_pinsker}
  \TV\left(s_\tau \left\| \target \right.\right) &\leq&  \frac{\delta}{1+\delta} +\sqrt{\frac{1}{2}\KL\left(\left.\frac{\widehat\target}{\widehat\target(\Statespace)} \right\| \target \right)},
\end{eqnarray}
with $\delta :=  \int_{s\in\Statespace} \deltafinit(s)d\lambda(s)$, so that $\delta = (\Finit+\Finstar-\Foutstar)^-(\Statespace)\geq 0$.
The bound \eqref{equ:control_pinsker} also holds with $\KL\left(\target \,\|\, \widehat\target/\widehat\target(\Statespace)\right)$, the divergence decomposed below, in place of $\KL\left(\widehat\target/\widehat\target(\Statespace) \,\|\, \target\right)$.
\end{corollary}

\begin{algorithm}
\caption{Ergodic flow: IL training}\label{alg:IL_EGF}
\begin{algorithmic}
\STATE {\bfseries Input:} A list of trainable bi-Lipschitz maps $\Phi_1,\cdots, \Phi_p$
\STATE {\bfseries Input:} A softmax model $\alpha:\Statespace \rightarrow \mathcal P(\{1,\cdots,p\})$
\STATE {\bfseries Input:} A trainable star outflow  model $f^*_\rightarrow:\Statespace\rightarrow \mathbb R_+$
\STATE {\bfseries Input:} A target samplable distribution $\target$
\STATE {\bfseries Input:} A source $\Finit$  samplable and of  density  $\finit$
 
\REPEAT
\STATE Fill replay buffer $\mathcal B=\mathcal B_{\rightarrow} \cup \mathcal B_{\leftarrow}$ with $B$ trajectories $(\underline s_t)_{t=1}^{t_{\text{max}}}$ in each of $\mathcal B_{\rightarrow}$ and $\mathcal B_{\leftarrow}$, using $(\Finit, \ForwardPolicy^*)$ for $\mathcal B_{\rightarrow}$ and $(\kappa, \BackwardPolicy^*)$ for $\mathcal B_{\leftarrow}$ ($\mathcal B_{\leftarrow}$ holds paths of the sampler of the reversed generative flow $(\BackwardPolicyStar,\Finstar)$ started from $\kappa$; the reversed flow is flow matching from $\widehat\target$ to $\Finithat=\Finit+\delta\Finit$, and its sampler continues at $\underline s_t$ with probability $\big(1+\widehat f_{\mathrm{init}}/\finstar\big)^{-1}(\underline s_t)$, $\widehat f_{\mathrm{init}}:=\finit+\deltafinit$, as the forward sampler continues with probability $\big(1+\ftermhat/\foutstar\big)^{-1}$)
\STATE Minimization step of
$$ 
 \mathcal L=\frac{b}{2B}\sum_{\underline s\in\mathcal B}\sum_{t=1}^{t_{\text{max}}}\deltafinit(\underline s_t)p_t(\underline s)-\mathbb E_{s\sim\target}\log\ftermhat(s),
$$
a Monte-Carlo estimate of \eqref{equ:weakFM} with $b$ replaced by $bm$, for $\nu_{\mathrm{train}}$ the $p_t$-weighted law of the states of the buffer's paths and $m\geq1$ the expected mean over the buffer's paths of $\sum_{t\leq t_{\text{max}}}p_t$ ($p_1=1$), with $b\geq1+\|d\lambda/d\nu_{\mathrm{train}}\|_\infty$ (so that $bm\geq b$), and
with 
\begin{eqnarray*}
   p_t(\underline s) &=&\prod_{t'<t}  \left(1+(\ftermhat/\foutstar)(\underline s_{t'})\right)^{-1}\quad \text{if } \underline s\in \mathcal B_{\rightarrow}\\
     &=&\prod_{t'<t}  \left(1+(\widehat f_{\mathrm{init}}/\finstar)(\underline s_{t'})\right)^{-1}\quad \text{if } \underline s\in \mathcal B_{\leftarrow} 
\end{eqnarray*}
where $\deltafinit\geq 0$ is the negative part of $\finit+\finstar-\foutstar$, as in \eqref{equ:weakFM}, so that each factor of $p_t$ lies in $(0,1]$ where $\foutstar>0$, resp.\ $\finstar>0$.
\UNTIL{converged}
\end{algorithmic}
\end{algorithm}

First, the weak-FM term $\mathbb E_{s\sim \nu_{\mathrm{train}}}\deltafinit(s)$ in equation \eqref{equ:weakFM} controls the term $\frac{\delta}{1 + \delta}$ in equation \eqref{equ:control_pinsker}.

Second, since the target $\target$ has unknown density, we employ reward-less inference, using $\Fterm = (\Finit+\Finstar-\Foutstar)^+$.
It is therefore natural to directly train the density of  $(\Finit+\Finstar-\Foutstar)^+$ to match $\target$ using cross-entropy, leading to the second term on the right-hand side of Equation \eqref{equ:weakFM}.
Corollary \ref{cor:pinsker_negative_control} demonstrates that controlling the Kullback-Leibler divergence through cross-entropy also controls the total variation sampling error.

Third, more precisely, the KL term may be decomposed into
\begin{multline}
    \KL\left(\target \left\| \frac{\widehat\target}{\widehat\target(\Statespace)}\right.\right) =\\ \underbrace{-\mathbb E_{s\sim \target}\log \ftermhat(s)}_{\text{cross-entropy}} + \underbrace{\log \int_{s\in \Statespace} \ftermhat(s)d\lambda(s)}_{\text{normalization}} -  \mathcal H(\kappa), 
\end{multline}
where $\mathcal H(\kappa)=-\mathbb E_{s\sim \kappa} \log\frac{d\kappa}{d\lambda}(s) $ is the entropy of $\kappa$, assumed finite.
Since the normalization constant is unknown a priori,  the cross-entropy is insufficient to fully control the KL divergence. However, 
\begin{equation}
    \int_{s\in \Statespace} \ftermhat(s)d\lambda(s) =1 + \int_{s\in \Statespace} \deltafinit(s)d\lambda(s),
\end{equation}
since $\int_{\Statespace}\ftermhat\,d\lambda-\int_{\Statespace}\deltafinit\,d\lambda=\Finit(\Statespace)+\Finstar(\Statespace)-\Foutstar(\Statespace)=1$.
Therefore, the weak-FM term in equation \eqref{equ:weakFM} also controls the normalization factor of $\ftermhat$.

\section{Experiments}
We proceed with experiments wherein the state space $\Statespace$ is either a flat torus $\mathbb T^2$, or sphere $\mathbb S^2$ (the size ablations of Appendix~\ref{sec:size_ablation} are run on $\mathbb R^2$).
An EGF family is built by randomly choosing diffeomorphisms from affine toroidal families on $\mathbb T^2$ and from isometry spherical families on $\mathbb S^2$,
together with two multi-layer perceptron (MLP) models: one for $\foutstar$ and one for $\forwardalpha$. The MLPs are $\tanh$ hyperbolic activated and initialized using orthogonal initialization \cite{saxe2013exact,hu2020provable}. We use the AdamW optimizer \cite{kingma2015adam,loshchilov2019international} for training.  

The KL-WeakFM loss tends to favor $\Ftermhat$ positive on the whole $\Statespace$, therefore, even with highly localized target distributions, the EGF reward model tends to have a small background inducing unwanted outliers. We filter it out by replacing $\ftermhat$ with $\mathbf{1}_{\ftermhat>\eta}\ftermhat$ during inference, where $\eta$ is chosen to minimize negative likelihood on a validation dataset. See Appendix \ref{sec:filter} for details.

\subsection{RL Experiments}
Two dimensional distributions are tested  in RL settings, which allow us to conduct sanity-check on  the tractability of the EGF, its stability and expressivity. An EGF on $\Statespace = \mathbb T^2$ is built with 16 transformations (8 translations and 2 elements of $\mathrm{SL}_d(\mathbb Z)$ together with their inverse). Their MLPs have 5 hidden layers of width 32 to parameterize $\foutstar$ and $\ForwardPolicyStar$. This EGF is trained either with stable $ \mathcal L_{\mathrm{FM}}^{\mathrm{stable}}$ or unstable $\mathcal L_{\mathrm{FM}}^{\mathrm{div}}$  FM-losses, with or without regularization  $\mathcal R$ (see equations~\eqref{equ:FMloss},\eqref{equ:RL_losses} and \eqref{equ:RL_reg}). 
\begin{eqnarray}\label{equ:RL_losses}
      \mathcal L_{\mathrm{FM}}^{\mathrm{div}} &=& \mathbb E_{\underline s}\sum_{t=1}^\tau \log\left(\frac{\finstar+\finit}{\foutstar+\fterm}\right)^2(\underline s_t) \\
            \mathcal R &=& \mathbb E_{\underline s}\sum_{t=1}^\tau (\foutstar)^2(\underline s_t) \label{equ:RL_reg}.
\end{eqnarray}
The regularization $\mathcal R$ is motivated by the stability Theory of \citet{brunswic2024theory}:
as long as the directional derivative of $\mathcal R$ is positive on 0-flows directions, such a regularization helps reducing the instability of a loss. See also \citet{morozov2025revisiting} for a study of the impact of such regularizations for detailed balance losses on graphs.

 A stability comparison to divergence-based FM loss (see figure~\ref{fig:RL_checkerboard}) shows the expressivity of very small EGF as well as the instability of $\mathcal L_{\mathrm{FM}}^{\mathrm{div}}$ as predicted by \citet{brunswic2024theory}(refer to limitation 4 in section \ref{sec:preliminaries}). 

\begin{figure}
    \centering
    \captionsetup[subfigure]{labelformat=simple}
    \begin{subfigure}{\linewidth}
        \centering
        \begin{subfigure}{0.48\linewidth}
            \centering
            \includegraphics[width=\linewidth]{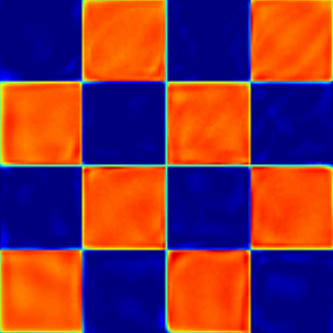}
        \end{subfigure}
        \hfill
        \begin{subfigure}{0.48\linewidth}
            \centering
            \includegraphics[width=\linewidth]{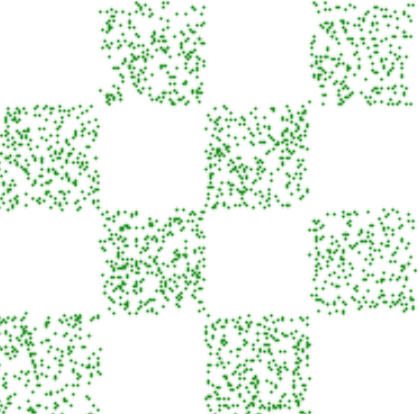}
        \end{subfigure}
        \caption{EGF 32x5, density (left) and samples (right).}
    \end{subfigure}
 \begin{subfigure}{\linewidth}
        \centering
        \begin{subfigure}{0.48\linewidth}
            \centering
            \includegraphics[width=\linewidth]{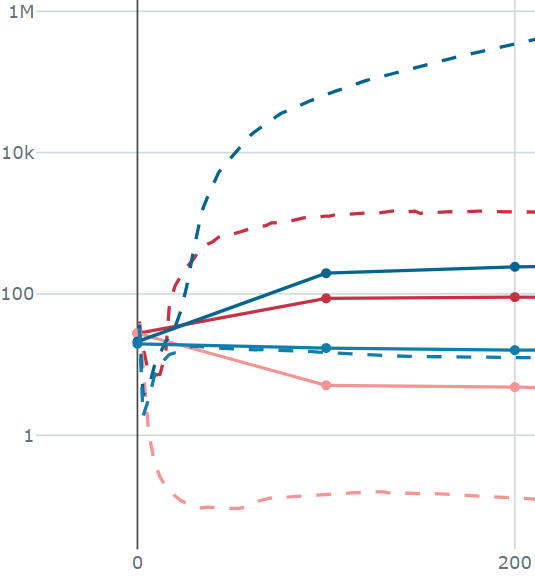}
        \end{subfigure}
        \hfill
        \begin{subfigure}{0.48\linewidth}
            \centering
            \includegraphics[width=\linewidth]{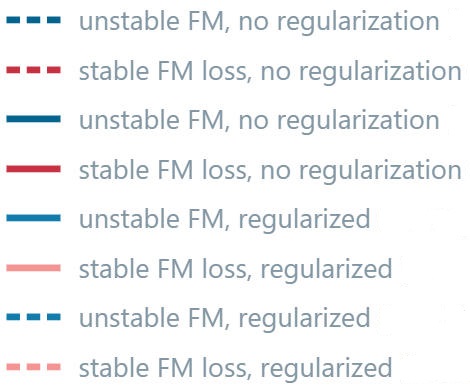}
        \end{subfigure}
        \caption{EGF 32x5, stability comparison between FM losses. Unregularized unstable loss (blue) flow size (dashed lines) and sampling time $\tau$ (solid lines) blow up while stable loss (red) converges. Regularization helps stabilization for both losses.}
    \end{subfigure}
    \caption{Checkerboard RL task.}
    \label{fig:RL_checkerboard}
\end{figure}

\subsection{IL Experiments}
We proceed with imitation learning experiments on tori $\mathbb T^2$ and spheres $\mathbb S^2$.

 On $\mathbb{T}^2$, we compare small Moser Flows, DDPM and EGFs, demonstrating that EGFs can generate well-behaved distributions even when Moser Flows break down.
 An EGF is built with the minimal four affine transformations described in section \ref{sec:universality_tori}, implemented as MLPs with 3 hidden layers of width 32. These transformations parameterize $\foutstar$ and $\ForwardPolicyStar$. The Moser Flow is trained using the implementation provided in the authors' GitHub repository \cite{mosergithub}  with the only modification being the model size set to 32x3. Our EGF is trained using the KL-weakFM loss together with a  regularization $\mathcal R$ as in the RL setting.
  Figure \ref{fig:toyimg} shows how Moser Flow fails to train with such a small model while minimal EGF reproduces the target distribution with high fidelity.

  On $\mathbb S^2$, we benchmark EGFs on the earth science volcano dataset \cite{volcano}. A sample distribution is given for the dataset in Figure \ref{fig:volcanos} and negative log-likelihoods are given in Table \ref{table:nll}. We only use six rotations: a rotation of angle $\pi/4$ around each of the three axes plus their respective inverse.  The two core MLPs of EGF are of size 256x5, compared to the 512x6 used by \citet{rozen2021moser}. The learning rate is 1e-3 with an exponential decay to 1e-5 at 3000 epochs of 25 steps.  We outperform Moser Flow and all related baselines. Notably, the EGF achieved its reported performance with a training time 10 times shorter.

\begin{figure}
    \centering
    \captionsetup[subfigure]{labelformat=simple}
    \begin{subfigure}{\linewidth}
        \centering
        \begin{subfigure}{0.45\linewidth}
            \centering
            \includegraphics[width=\linewidth,trim={1cm 1cm 1cm 1cm},clip]{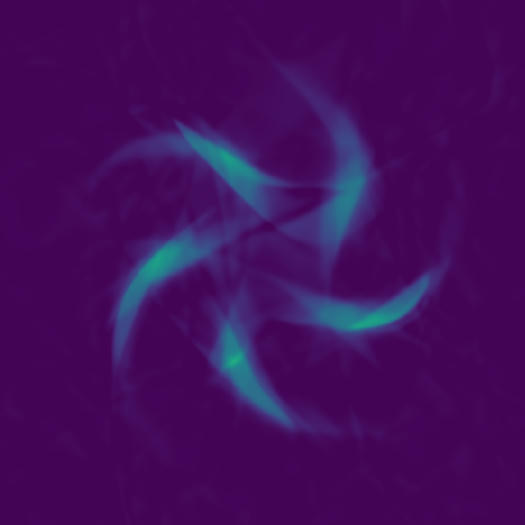}
        \end{subfigure}
        \hfill
        \begin{subfigure}{0.45\linewidth}
            \centering
            \includegraphics[width=\linewidth]{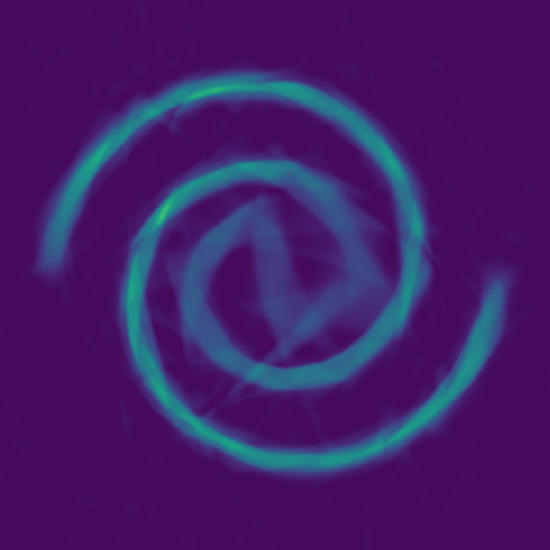}
        \end{subfigure}
        \caption{Moser Flow}
    \end{subfigure}
     \begin{subfigure}{\linewidth}
        \centering
        \begin{subfigure}{0.45\linewidth}
            \centering
            \includegraphics[width=\linewidth,trim={0cm 0cm 0cm 0cm},clip]{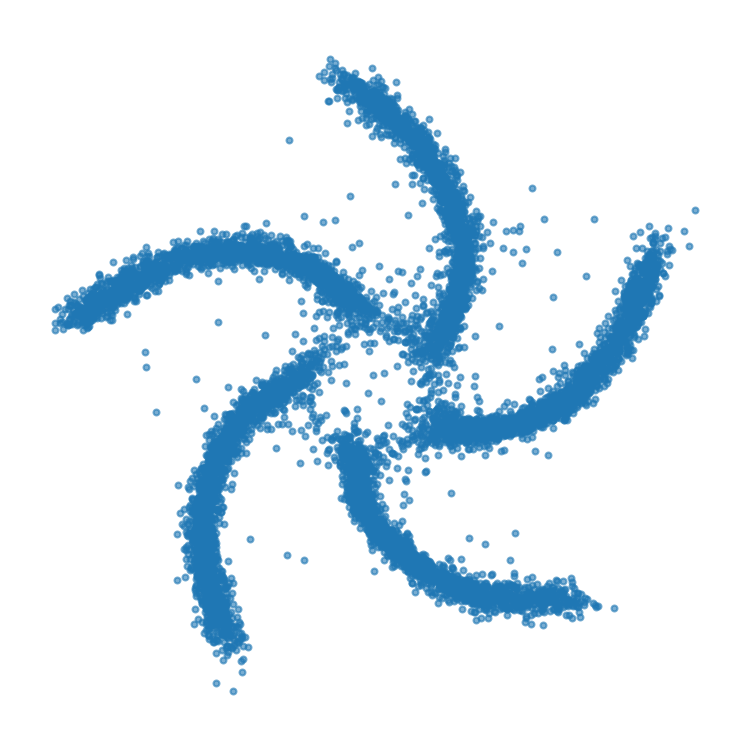}
        \end{subfigure}
        \hfill
        \begin{subfigure}{0.45\linewidth}
            \centering
            \includegraphics[width=\linewidth,trim={0cm 0cm 0cm 0cm},clip]{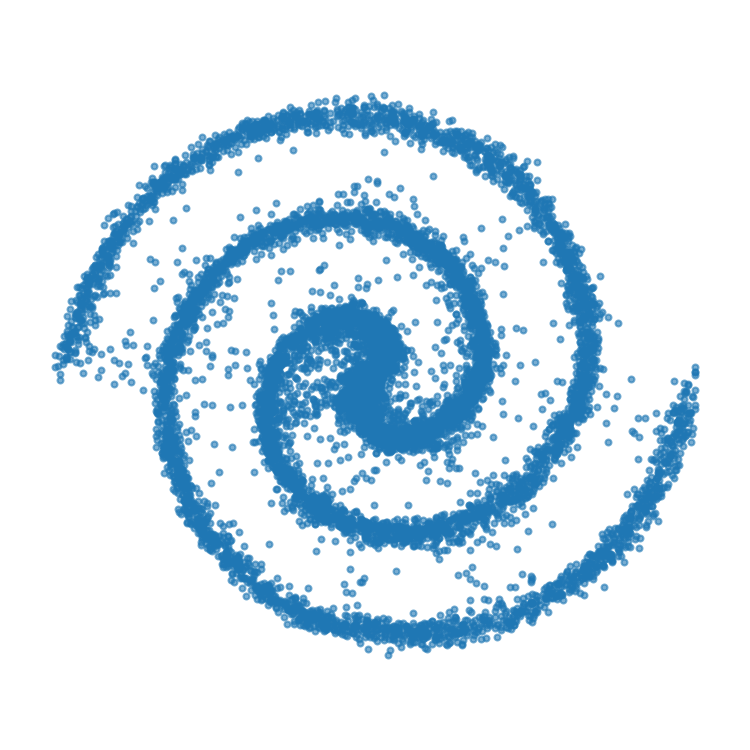}
        \end{subfigure}
        \caption{DDPM}
    \end{subfigure}
    \begin{subfigure}{\linewidth}
        \centering
        \begin{subfigure}{0.45\linewidth}
            \centering
            \includegraphics[width=\linewidth,trim={7cm 7cm 7cm 7cm},clip]{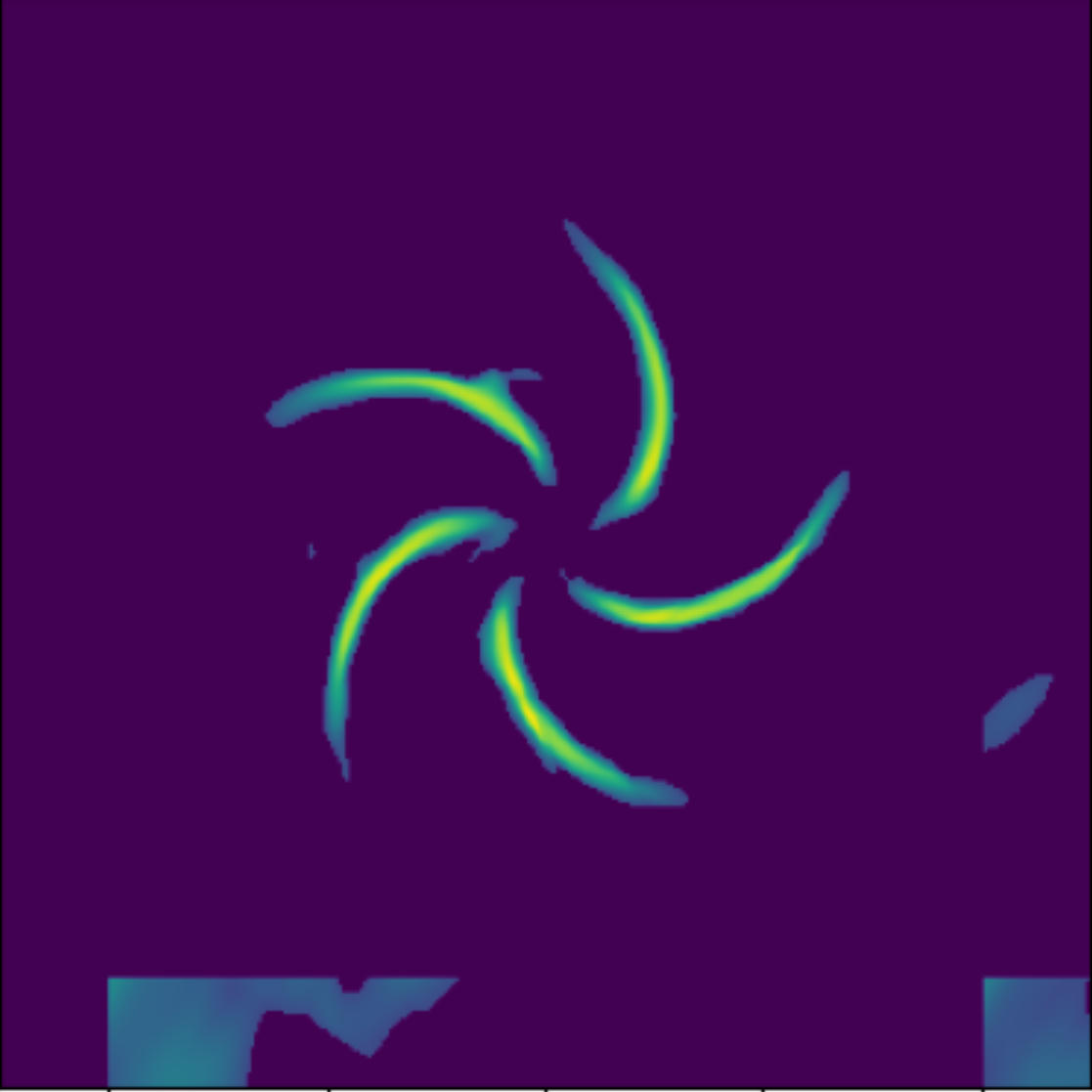}
        \end{subfigure}
        \hfill
        \begin{subfigure}{0.45\linewidth}
            \centering
            \includegraphics[width=\linewidth,trim={4cm 4cm 4cm 4cm},clip]{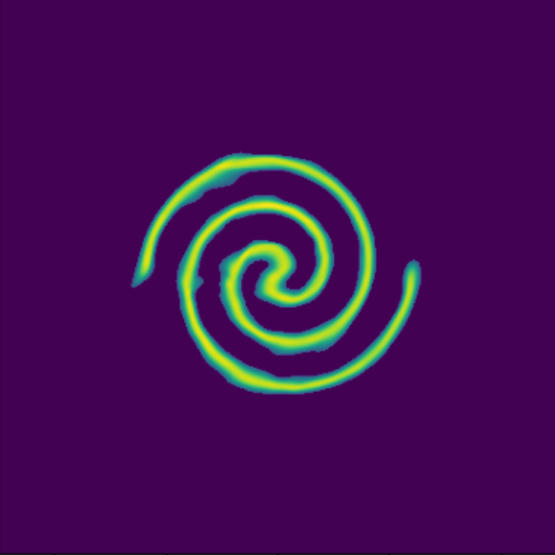}
        \end{subfigure}
        \caption{EGF}
    \end{subfigure}\vspace{-0.3cm}
    

    \caption{Comparison of imitation learning on standard toy distributions using tiny 32x3 MLP models. Background filter is applied to EGF, a similar filter on Moser flow would yield worse results. Without a filter, EGF samples have outliers similar to DDPM. Since DDPM does not give access to density, only samples are provided. For fairness, DDPM is sampled using 100 denoising steps.}
    \label{fig:toyimg} \vspace{-0.5cm}
\end{figure}

\begin{figure}
    \centering
    \captionsetup[subfigure]{labelformat=simple}
        \centering
        \begin{subfigure}{0.45\linewidth}
            \centering
            \includegraphics[width=\linewidth,trim={14cm 6cm 14cm 7cm},clip]{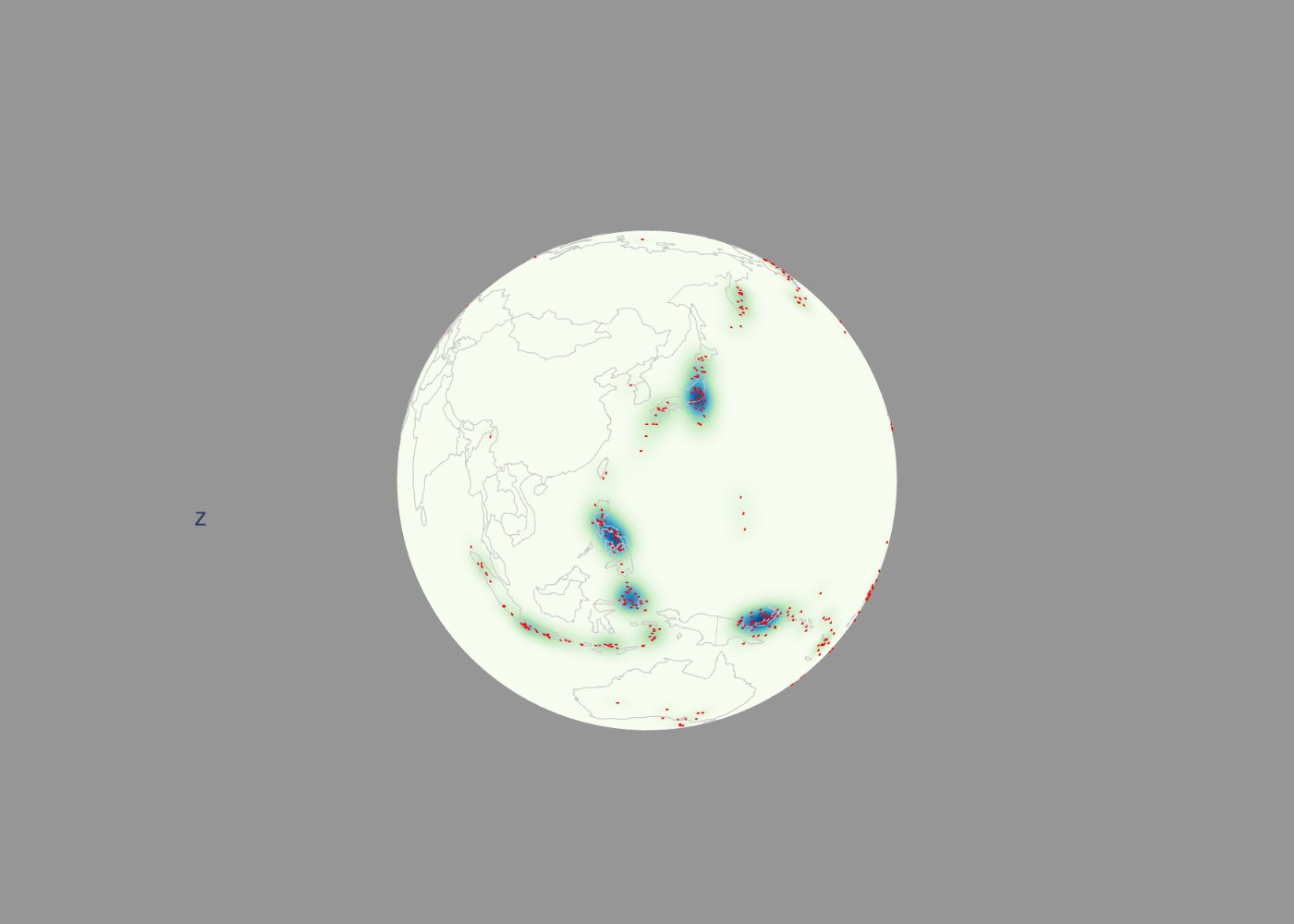}
        \end{subfigure}
        \hfill
        \begin{subfigure}{0.45\linewidth}
            \centering
            \includegraphics[width=\linewidth,trim={14cm 6cm 14cm 7cm},clip]{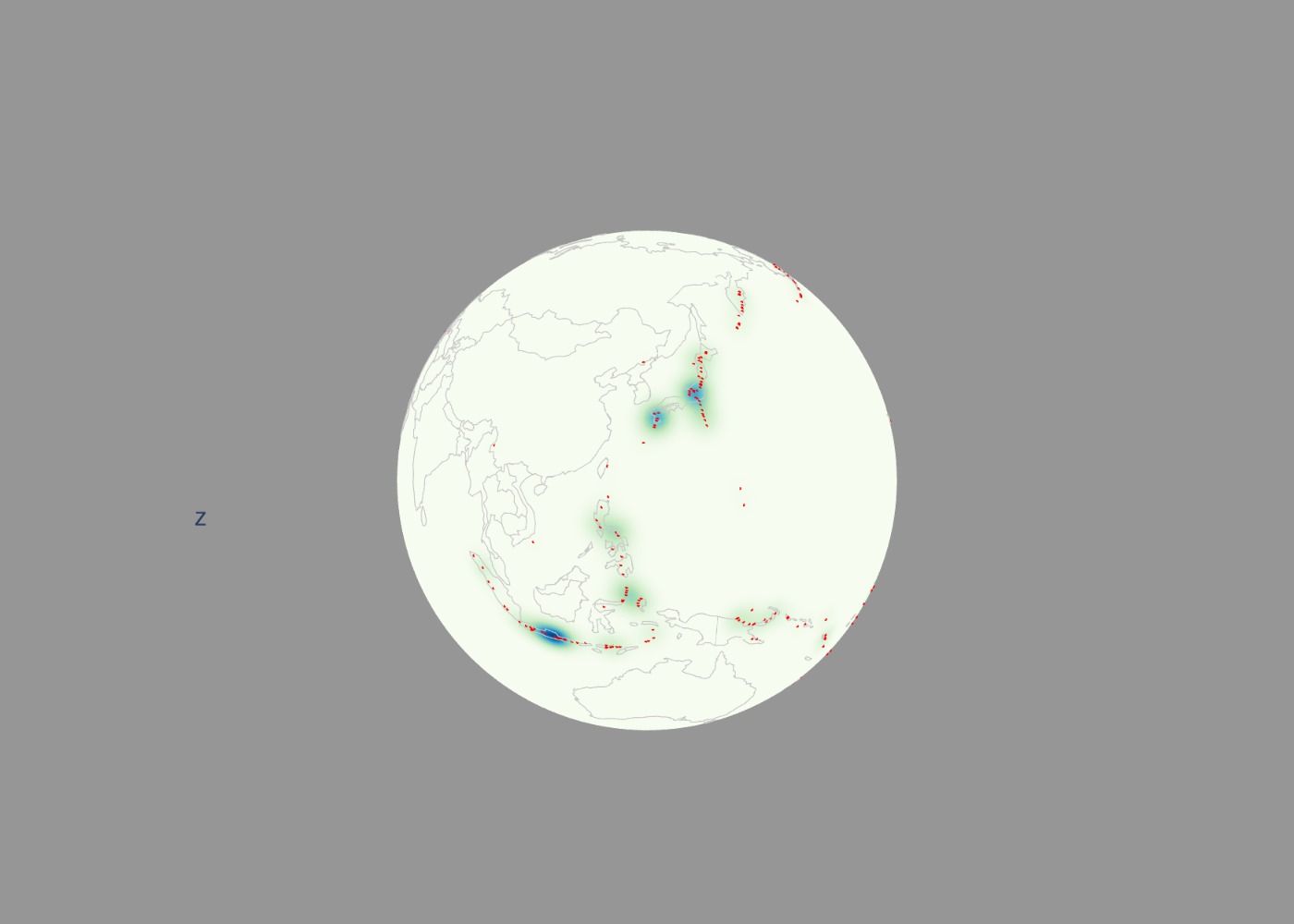}
        \end{subfigure}\vspace{-0.3cm}
    \caption{EGF generated points with reward field $\ftermhat$ (left) and whole dataset point with KDE field (right) for the volcano task.}
    \label{fig:volcanos}
\end{figure}

\begin{table}
    \caption{Negative log-likelihood scores of the volcano dataset.}
    \centering
    \begin{tabular}{c c c c}
       & Volcano
       & Earthquake & Flood
       \\
     \hline
     Mixture vMF & -0.31  
     & 0.59 & 1.09 
     \\
     Stereographic & -0.64 
     & 0.43 & 0.99
     \\
     Riemannian & -0.97 
     & 0.19 & 0.90
     \\
     Moser Flow & -2.02 & -0.09 & 0.62 
     \\
     EGFN & \textbf{-2.31} & \textbf{-0.12} & \textbf{0.56} 
     \\
     \hline
    \end{tabular}
    \label{table:nll}
\end{table}

\section{Related Work}
\label{Related Work}
EGFs draw inspiration from NFs for their transformation blocks, as well as from denoising diffusion models (DDMs) \cite{ho2020denoising} for incorporating stochasticity. While similar approaches have been explored in previous works \cite{wu2020stochastic, zhang2021diffusion}, our method stands out by leveraging general generative flows theory to construct a more efficient framework. Compared to previous works, EGFs are more compact and demonstrate faster convergence as illustrated in our experiments. Furthermore, DDMs and their variants \cite{song2020denoising} rely on a fixed discretization of inference denoising trajectories, whereas EGFs enable trainable sampling times, offering greater flexibility and adaptability.

A particularly notable comparison arises when examining the relationship between our EGFs and NFs introduced in  \citet{rezende2020normalizing}, as both utilize affine toroidal families and projective spherical families. However, while our EGF framework provides universality guarantees, the NF setting lacks such guarantees as it solely relies on affine and projective transforms. Indeed, both projective and affine transformations form groups (compositions of affine transforms are affine), thus limiting the expressivity. This distinction highlights a fundamental advantage of EGFs, further reinforcing their theoretical robustness and practical effectiveness. 

Broadly speaking, previous methods, such as Riemannian Continuous Normalizing Flows \cite{mathieu2020riemannian} and Flow Matching on Manifolds \cite{chen2024flow} rely on either handcrafted non-linear transformations or sophisticated NeuralODEs \cite{chen2018neural}. Our EGF does not incorporate NeuralODEs, as we do not face expressivity limitations that justify their usage. However, these techniques remain fully compatible with EGFs and could be integrated into future extensions of our framework to further enhance its flexibility and adaptability if needed.

Since EGFs' elementary transformation are derived from the transformation blocks of NFs, EGFs can be viewed as a stochastic sampler for NF architectures. In particular, any EGFs sampled trajectory can be defined as $(s_t)_{t=1}^\tau$, where the trajectory $(s_t)_{t=1}^\tau$ is obtained by composing NF transformations as follows: $s_1 \sim \mathcal{F}_{\text{init}}, \quad s_2 = \Phi_{k_1}(s_1), \quad \dots, \quad s_\tau = \Phi_{k_{\tau-1}} \circ \cdots \circ \Phi_{k_1}(s_1)$,  
with the transition probability $\mathbb{P}(k_{t+1} \mid k_t) = \forwardalpha^i(s_t)$. Each composition $\Phi_{k_{\tau-1}} \circ \cdots \circ \Phi_{k_1}$ is a diffeomorphism, and the terminal distribution is given by $\Fterm=\mathbb E_K(\Finit K)$, where $K=\Phi_{k_{\tau-1}} \circ \cdots \circ \Phi_{k_1}$ represents a random NF.
In other words, an EGF with $p$ diffeomorphisms can be represented as a pair $(\iota, \pi)$, where $\iota : \mathcal T_p \rightarrow \mathrm{Diffeo}(\Statespace)$ is a trainable embedding of the $p$-ary tree into the space of diffeomorphisms on the state space, and $\pi$ is a random walk policy on $\mathcal T_p$ with a random stopping time.
However, the policy $\pi$ is intractable as it is defined by $\pi(\Phi_{k} | \Phi_{k_t}\circ \cdots \Phi_{k_1}) = \frac{1}{\Fout(\Statespace)}\int_{s\in\Statespace} \forwardalpha^k(s)d\Fout(s)$, where $\Fout := \Foutstar + \Fterm$.

Building on this formulation, we can draw comparisons between EGFs and three related research directions: Continuously Indexed Normalizing Flows (CINF) \cite{caterini2021variational}, Neural Architecture Search (NAS) \cite{elsken2019neural}, and Wasserstein Gradient Descent (WGD) \cite{NEURIPS2018_a1afc58c}. First, CINFs attempt to overcome the limitations of NFs by using a fixed architecture, with conditioning sampled from a latent distribution. This approach generates the target distribution as an expectation, $\Fterm=\mathbb E_K(\Finit K)$, where the random NF $K$ is drawn from a continuous distribution of NFs. In contrast, our method samples $K$ from a tree structure. Second, NAS aims to find the optimal neural network architecture for a given task. In comparison, our EGF constructs a distribution of suitable NFs, where the sampling time can be interpreted as a learned depth of the resulting architecture. Third, by training over a distribution of architectures and considering the expectation of the output (with $\Finit$ as the input and $\Fterm$ as the output), our approach parallels the setup of Wasserstein Gradient Descent (WGD). Extending the theorems of \citet{NEURIPS2018_a1afc58c} to EGFs would be a valuable addition to our theoretical framework.

Lastly, when $\lambda$ is invariant, $\ForwardPolicyStar$ is summably $L^2$-mixing with $\widehat\lambda=\lambda$ and $\Finit,\Fterm$ have bounded densities (the remark after the proof of Theorem~\ref{theo:universality_true}), $\foutstar$ can be obtained from the partial sums of a series whose sum is a fixed point of the flow operator $\nu\mapsto\nu\ForwardPolicyStar+\Finit-\Fterm$, some power of which is a contraction on mean-zero densities. Although we train $\foutstar$ via gradient descent, this approach bears similarities to Deep Equilibrium Models (DEQ) \cite{bai2019deep, bai2020multiscale}, where the fixed point is computed by a contraction neural network. The key difference is that we explicitly approximate the fixed point with a neural network, while DEQ models use a black-box solver and the implicit function theorem for gradients. Future work could develop a DEQ-EGF, where $\foutstar$ is implicitly obtained as a fixed point instead of via a feedforward network.

\section{Limitations and Future Work}

First and foremost, our experiments are limited to low dimensions. Although ergodicity is easily achievable with a number of generators independent from the dimension (for instance on tori, one randomly initialized non-trainable translation is sufficient to ensure ergodicity), the technical assumption of $L^2$-mixing summability, which universality does not need (Theorem~\ref{theo:universality_true}) but under which, for $\lambda$ invariant and $\widehat\lambda=\lambda$, the outflow is a convergent series (the remark after its proof), is more subtle.
EGFs suffer from limited proven applicability, further theory is needed to easily enforce $L^2$-mixing summability in higher dimension together with related experiments. More precisely, a theoretical bound on the minimal number of transformations necessary to achieve universality would be useful for hyperparameter tuning. We give such a bound only for affine toroidal and isometry spherical families, where two transformations are sufficient to achieve universality: Theorems~\ref{theo:universality_tori} and \ref{theo:universality_sphere} ask of the moves only that the group they generate be ergodic, respectively dense, so the two inverse moves among the four of Section~\ref{sec:universality} are not needed, and on $\mathbb T^d$, $d\geq2$, a single affine move whose linear part has no periodic integer vector is already sufficient.


Second, the $L^2$-mixing summability property is not needed for universality (Theorem~\ref{theo:universality_true}). It is likely that training would be improved if some control over this property was to be achieved, say by adding a regularization. There is extensive mathematical literature on the so-called spectral gap \cite{kontoyiannis2012geometric,guivarc2016spectral,marrakchi2018strongly,bekka2020spectral}, which, in the two-sided form of Section~\ref{sec:universality_master} ($\lambda$ invariant and spectral radius $<1$ on mean-zero functions), is a sufficient condition for $L^2$-mixing summability, and absolutely continuous invariant measures \cite{gora2003absolutely,bahsoun2004absolutely,galatolo2015statistical} including stability results \cite{froyland2014stability}. A systematic review of this literature would certainly yield such regularization.

Lastly, our focus was on the theoretical advancements of generative flow theory through EGF. As a result, we did not exploit EGF's high modularity, which enables sophisticated transformations, replay buffers, or neural architectures. Additionally, hyperparameter tuning was minimal, leaving room for future work to perform a systematic analysis.



\section{Conclusion}
%
We propose a new family of generative flows that leverage ergodicity to provide universality guarantees while utilizing simple diffeomorphisms and neural networks. Our results demonstrate that EGFs maintain their expressivity even at parameter counts where Moser Flow \cite{rozen2021moser} fails. This simplicity enables us to outperform our baselines on the NASA volcano dataset with a model 30 times smaller.

\section*{Impact Statement}
This paper presents work whose goal is to advance the field of Machine Learning. There are many potential societal consequences of our work, none which we feel must be specifically highlighted here.
\bibliography{note}
\bibliographystyle{icml2025}

\newpage
\appendix
\onecolumn
\section{Fundamentals of EGFs} \label{sec:appendix_backward}

As in Section~\ref{sec:fundamentals_EGF}, but without the ergodicity clause of Definition~\ref{def:egf}, which this appendix does not use, let $\Statespace$ be a differentiable manifold endowed with a background measure $\lambda$ for which every $\Phi_i$ is nonsingular, $(\Phi_i)_\#\lambda \ll \lambda$ (for instance the Riemannian volume, or any $\lambda$ equivalent to Lebesgue measure), and consider a Generative Flow $(\ForwardPolicyStar, \Foutstar)$ with
$\Foutstar = \foutstar \lambda$ and $\ForwardPolicyStar(s\rightarrow \cdot) = \sum_{i=1}^p\forwardalpha^{i}(s)\delta_{\Phi_i(s)}$, where the $\Phi_i$ are diffeomorphisms of $\Statespace$ and $\forwardalpha$ is a measurable policy with $\sum_i\forwardalpha^i=1$. With these definitions, the backward policy automatically has a closed form formula.
\begin{theorem} Let $(\ForwardPolicyStar,\Foutstar)$ as above, the induced Generative Flow  on $\Statespace$ has a backward star policy given by  
    \begin{eqnarray}\label{equ:tractable_inflow_appendix}
\BackwardPolicy^*(s) &=& \sum_{i=1}^p\backwardalpha^i(s) \delta_{\Phi_i^{-1}(s)} \\ 
\backwardalpha^i(s) &=& \frac{(\forwardalpha^if^*_\rightarrow)\circ \Phi_i^{-1}(s)\, J^\lambda_{\Phi_i^{-1}}(s) }{ \finstar(s)} \\
\finstar &=&  \sum_j (\forwardalpha^jf^*_\rightarrow)\circ \Phi_j^{-1}(s)\, J^\lambda_{\Phi_j^{-1}}(s).
\end{eqnarray}
where $J^\lambda_{\Phi_i^{-1}}(s) := \frac{d((\Phi_i)_\#\lambda)}{d\lambda}(s)$ is the Jacobian of $\Phi_i^{-1}$ relative to $\lambda$, which is $|\det J_s\Phi_i^{-1}|$ for Lebesgue (or Haar) measure and the Riemannian Jacobian for the volume measure.
$\backwardalpha^i$ is defined $\Finstar$-almost everywhere (set to $0$ where $\finstar(s)\in\{0,\infty\}$); the backward policy is unique up to $\Finstar$-null sets, setwise: for each measurable $X$, every backward star policy agrees with $\BackwardPolicyStar(s\rightarrow X)$ for $\Finstar$-almost every $s$.
\end{theorem}
\begin{proof}
    From \cite{brunswic2024theory} section 3.2 and Proposition 1 in their appendix, we have for any measurable $X,Y\subset \Statespace$:
    \begin{eqnarray}
        \Finstar(Y) &=&  \Foutstar \otimes \ForwardPolicyStar (\Statespace \rightarrow Y) \\
        \BackwardPolicyStar(s\rightarrow X) &=&  \frac{d \Finstar \otimes \BackwardPolicyStar ( \cdot \rightarrow X)}{d\Finstar}(s)    \\
 \Finstar \otimes \BackwardPolicyStar(Y\rightarrow X) &=&  \Foutstar \otimes \ForwardPolicyStar(X\rightarrow Y). 
\end{eqnarray}
Therefore, for all measurable $X,Y\subset \Statespace$:
\begin{eqnarray}
\Finstar \otimes \BackwardPolicyStar (Y\rightarrow X) &=&  \Foutstar \otimes \ForwardPolicyStar (X\rightarrow Y) \\ 
     &=& \int_{s\in \Statespace} \mathbf{1}_{X}(s) \ForwardPolicyStar(s\rightarrow Y) d\Foutstar(s) \\ 
    &=& \int_{s\in \Statespace}\mathbf{1}_{X}(s)  \ForwardPolicyStar(s\rightarrow Y) \foutstar(s) d\lambda(s)\\ 
    &=& \int_{s\in \Statespace} \sum_{i=1}^p \mathbf{1}_{X}(s) \forwardalpha^i(s) \delta_{\Phi_i(s)}(Y) \foutstar(s) d\lambda(s)\\ 
    &=& \sum_{i=1}^p\int_{s\in \Statespace} \mathbf{1}_{X}(s)  \forwardalpha^i(s) \mathbf{1}_{Y}(\Phi_i(s)) \foutstar(s) d\lambda(s)\\ 
    &=& \sum_{i=1}^p\int_{u\in \Statespace} \mathbf{1}_{X}(\Phi_i^{-1}(u)) \forwardalpha^i(\Phi_i^{-1}(u)) \mathbf{1}_{Y}(u) \foutstar(\Phi_i^{-1}(u)) J^\lambda_{\Phi_i^{-1}}(u) d\lambda(u)\\  
    &=& \sum_{i=1}^p\int_{u\in Y} \mathbf{1}_{X}(\Phi_i^{-1}(u))  \forwardalpha^i(\Phi_i^{-1}(u)) \foutstar(\Phi_i^{-1}(u)) J^\lambda_{\Phi_i^{-1}}(u) d\lambda(u)\\  
    &=&\int_{u\in Y} \sum_{i=1}^p  \delta_{\Phi_i^{-1}(u)}(X)  \forwardalpha^i(\Phi_i^{-1}(u)) \foutstar(\Phi_i^{-1}(u)) J^\lambda_{\Phi_i^{-1}}(u) d\lambda(u).  
\end{eqnarray}
Since this formula is true for any measurable $Y\subset \Statespace$, we deduce that $\Finstar \otimes \BackwardPolicyStar(\cdot \rightarrow X)\ll \lambda$ and that 
\begin{eqnarray}
    \frac{d\Finstar \otimes \BackwardPolicyStar(\cdot \rightarrow X)}{d\lambda}(x) &=& \sum_{i=1}^p  \delta_{\Phi_i^{-1}(x)}(X)  \forwardalpha^i(\Phi_i^{-1}(x)) \foutstar(\Phi_i^{-1}(x)) J^\lambda_{\Phi_i^{-1}}(x)\\
    &=& \left(\sum_{i=1}^p   \forwardalpha^i(\Phi_i^{-1}(x)) \foutstar(\Phi_i^{-1}(x)) J^\lambda_{\Phi_i^{-1}}(x) \delta_{\Phi_i^{-1}(x)} \right) (X)\\
    \finstar(x) &=&  \frac{d\Finstar \otimes \BackwardPolicyStar(\cdot \rightarrow \Statespace)}{d\lambda}(x) \\&=&   \sum_{i=1}^p  \delta_{\Phi_i^{-1}(x)}(\Statespace)  \forwardalpha^i(\Phi_i^{-1}(x)) \foutstar(\Phi_i^{-1}(x)) J^\lambda_{\Phi_i^{-1}}(x)\\
    &=&\sum_{i=1}^p   \forwardalpha^i(\Phi_i^{-1}(x)) \foutstar(\Phi_i^{-1}(x)) J^\lambda_{\Phi_i^{-1}}(x)
\end{eqnarray}
The result follows.
\end{proof}

\section{Universality Theorems}
\subsection{General Universality Theorem}
We shall both write the theorems in a more formal way and provide rigorous proof of each statements. 

The following theorem implies Theorem \ref{theo:universality}.
\begin{theorem} \label{theo:universality_true}Let $(\Statespace,\lambda)$ be a measured space with $\lambda$ finite and $\ForwardPolicyStar$ a Markov kernel on $\Statespace$ which admits a finite invariant measure $\widehat\lambda$ equivalent to $\lambda$ for which it is ergodic. Then for any probability distributions $\Finit,\Fterm\ll\lambda$ and any $\varepsilon>0$ there is a bounded $\foutstar\geq0$ such that the generative flow $(\ForwardPolicyStar,\foutstar)$ from $\Finit$ to $\Fterm$ has $\delta\Finit(\Statespace)+\delta\Fterm(\Statespace)<\varepsilon$. Conversely, if $\ForwardPolicyStar$ is not ergodic for $\widehat\lambda$, some such pair has $\delta\Finit(\Statespace)+\delta\Fterm(\Statespace)\geq2$ for every $\foutstar$.
\end{theorem}
When $\lambda$ itself is invariant, summable $L^2$-mixing with $\widehat\lambda=\lambda$ (Definition~\ref{def:summably_mixing}) gives the outflow as a convergent series (see the remark after the proof); universality itself needs only ergodicity.

\begin{proof}
    \emph{Sufficiency.} Let $\varepsilon>0$ and $\eta=\varepsilon/5$; norms of signed measures are total variation norms. Truncating their densities, choose probability measures $\mu_0=a\lambda$, $\mu_1=b\lambda$ with $a,b\leq K\frac{d\widehat\lambda}{d\lambda}$ for some $K<\infty$ and $\|\Finit-\mu_0\|,\|\Fterm-\mu_1\|<\eta$. Since $\widehat\lambda\ForwardPolicyStar=\widehat\lambda$, every $\mu_j(\ForwardPolicyStar)^t$ is at most $K\widehat\lambda$. The operator $P:h\mapsto\frac{d((h\widehat\lambda)\ForwardPolicyStar)}{d\widehat\lambda}$ is a contraction of $L^2(\widehat\lambda)$ whose fixed points are constant, $\ForwardPolicyStar$ being ergodic for $\widehat\lambda$; by von Neumann's mean ergodic theorem, $\frac1N\sum_{t<N}P^th\to0$ in $L^2(\widehat\lambda)$, hence in $L^1(\widehat\lambda)$, for the mean-zero $h=\frac{d(\mu_0-\mu_1)}{d\widehat\lambda}$, so that some $N\geq1$ has $\big\|\sum_{t<N}(\mu_0-\mu_1)(\ForwardPolicyStar)^t\big\|<N\eta$. Let $R=NK\widehat\lambda$, which is invariant and dominates every $\sum_{t\leq T}\mu_1(\ForwardPolicyStar)^t$, $T<N$, and
    \[F=\frac1N\sum_{T<N}\Big(\sum_{t\leq T}\mu_0(\ForwardPolicyStar)^t+R-\sum_{t\leq T}\mu_1(\ForwardPolicyStar)^t\Big)\geq0,\]
    a Ces\`aro-weighted Neumann series. Telescoping, $\mu_0+F\ForwardPolicyStar-\mu_1-F=\frac1N\sum_{T<N}(\mu_0-\mu_1)(\ForwardPolicyStar)^{T+1}$, of norm less than $\eta$, the kernel being a contraction for the total variation norm. As $F\ll\lambda$, truncating its density at a level $M$ with $\|F-F_M\|<\eta$ gives a bounded $\foutstar=\frac{dF_M}{d\lambda}$, and $\delta\Finit(\Statespace)+\delta\Fterm(\Statespace)=\|\Finit+\Foutstar\ForwardPolicyStar-\Fterm-\Foutstar\|<5\eta=\varepsilon$.

    \emph{Necessity.} If $A$ is measurable with $\ForwardPolicyStar(x\rightarrow A)=\mathbf 1_A(x)$ for $\widehat\lambda$-almost every $x$ and $0<\lambda(A)<\lambda(\Statespace)$, take $\Finit=\lambda_{|A}/\lambda(A)$ and $\Fterm=\lambda_{|A^c}/\lambda(A^c)$. Every finite $\Foutstar\ll\lambda$ has $(\Foutstar\ForwardPolicyStar)(A)=\Foutstar(A)$, so the defect $\Finit+\Foutstar\ForwardPolicyStar-\Fterm-\Foutstar$ has mass $1$ on $A$ and $-1$ on $A^c$, and $\delta\Finit(\Statespace)+\delta\Fterm(\Statespace)\geq2$.
\end{proof}
\addtocounter{equation}{6}

\textit{Remark.} If $\lambda$ itself is invariant, $\lambda\ForwardPolicyStar=\lambda$, and $\ForwardPolicyStar$ is summably $L^2$-mixing with $\widehat\lambda=\lambda$, the Ces\`aro weights can be dropped. The coefficients are then submultiplicative, $\gamma_{m+n}\leq\gamma_m\gamma_n$, so $\sum_n\gamma_n<\infty$ gives $\gamma_N<1$ for some $N$, and $\gamma_n\leq C\rho^n$ for some $C<\infty$ and $\rho<1$. For $\Finit,\Fterm$ with bounded densities the series $\sum_{t\geq0}(\Finit-\Fterm)(\ForwardPolicyStar)^t$ then converges absolutely in $L^2(\lambda)$, and its partial sums, lifted by a multiple of $\lambda$, give the outflow directly (for general $\Finit,\Fterm\ll\lambda$, after truncating their densities as in the proof).

\subsection{Universality on Tori}
Theorem~\ref{theo:universality_tori} rests on the ergodicity of the group generated by the $\Phi_i$; one sufficient condition for it is a spectral gap of that group's action, which the following result characterizes.
\begin{theorem}[Theorem 5 of \cite{bekka2015spectral}  ]
\label{theo:spectral_gap_tori}
    Let $H$ be a countable subgroup of $\mathrm{Aff}(\mathbb T^d)$. The following properties are equivalent:
    \begin{enumerate}
        \item[(i)]  The action of H on $\mathbb{T}^d$ does not have a spectral gap.
\item[(ii)] There exists a non-trivial $H$-invariant factor torus $\overline{\mathbb T}$ such that the projection of $H$ on $\mathrm{Aut}(\overline {\mathbb T})$ is amenable.
    \end{enumerate}
\end{theorem}
If condition $(ii)$ of Theorem~\ref{theo:spectral_gap_tori} is false for the group generated by the $(\Phi_i)_{i=1}^p$, then the action of that group has a spectral gap, hence is ergodic: a set $A$ invariant under the group would give the invariant mean-zero function $\mathbf 1_A-\lambda(A)$. Since each $\Phi_i$ keeps the Lebesgue measure $\lambda$ invariant, the policy $\forwardalpha^i=\frac1p$ is then ergodic for $\lambda$, and by master Theorem~\ref{theo:universality_true} the family is universal. This is the first sufficient condition of Theorem~\ref{theo:universality_tori}: `the group generated by $(\Phi_i)_{i=1}^p$ violates condition $(ii)$'.
In particular, if $d\geq2$ and the linear part of $H$ is the whole $\mathrm{SL}(d,\mathbb Z)$, projections such as the one in condition $(ii)$ are never amenable since  $\mathrm{SL}(d,\mathbb Z)$ contains a free group with two generators.

The other sufficient conditions of Theorem~\ref{theo:universality_tori} need no spectral gap. Let $h(x)=Mx+c$ be an element of the group and $A$ a measurable set with $\lambda\big(h^{-1}(A)\,\triangle\,A\big)=0$. The Fourier coefficients of $f=\mathbf 1_A$ satisfy $|\hat f(M^{T}k)|=|\hat f(k)|$ for every $k\in\mathbb Z^d$. If $M$ (equivalently $M^T$) has no periodic integer vector, the frequencies $(M^{T})^nk$, $n\geq0$, are pairwise distinct for $k\neq0$, and $\sum_k|\hat f(k)|^2<\infty$ forces $\hat f(k)=0$ for $k\neq0$: $\mathbf 1_A$ is almost everywhere constant, and $h$ is ergodic. On $\mathbb T^1$, a translation $x\mapsto x+c$ multiplies $\hat f(k)$ by $e^{2i\pi kc}$, which differs from $1$ for every $k\neq0$ when $c$ has infinite order, with the same conclusion. A set almost invariant under the $\Phi_i$ is almost invariant under every element of the group, the moves preserving $\lambda$, and an ergodic map has no almost invariant set other than null and conull ones, so one ergodic element makes the group ergodic.

\subsection{Universality on Spheres}
Theorem~\ref{theo:universality_sphere} rests on the ergodicity of the group generated by the $\Phi_i$, not on a spectral gap. For comparison with the spectral-gap results cited in Section~\ref{sec:universality_sphere}, we recall that a complex number $x\in\mathbb C$ is {\it algebraic} if there exists a non-zero polynomial $P\in\mathbb Z[X]$ with integral coefficients such that $P(x)=0$. The set of algebraic numbers is denoted $\overline{\mathbb Q}$. Bourgain and Gamburd prove the following.
\begin{theorem}[Theorem 1 of \citet{bourgain2012spectral}]
\label{theo:spectral_gap_sphere} 
 Let $g_1,\dots,g_k\in\mathrm{SU}(d)\cap\mathrm{Mat}_{d\times d}(\overline{\mathbb Q})$, $d\geq2$, generate a dense subgroup of $\mathrm{SU}(d)$. Then the operator $f\mapsto\sum_{j=1}^k\left(f(g_j\,\cdot)+f(g_j^{-1}\,\cdot)\right)$ on $L^2(\mathrm{SU}(d))$ has its spectrum on the orthogonal complement of the constants bounded above by some $c<2k$.
\end{theorem}
The analogous statement for rotations of the sphere fails for $\mathrm{SO}(2)$: the rotation of cosine $3/5$ and its inverse have rational entries and generate a dense subgroup, while their averaging operator has eigenvalues arbitrarily close to $1$ on mean-zero functions (on the functions $\mathrm{Re}\,\zeta^k$, $\zeta=x_0+ix_1$).

 If the group generated by $(\Phi_i)_{i=1}^p$ is dense in $\mathrm{SO}_{d+1}(\mathbb R)$, $d\geq1$, it acts ergodically on $\mathbb S^d$: $\mathrm{SO}_{d+1}(\mathbb R)$ acts transitively on $\mathbb S^d$, hence ergodically (by Fubini on $\mathrm{SO}_{d+1}(\mathbb R)\times\mathbb S^d$), and, for a finite Borel measure on a compact metric space preserved by the action, such as $\lambda$ on $\mathbb S^d$, a set almost invariant under a dense subgroup of a group acting continuously and ergodically is null or conull. The rotations preserve $\lambda$, so the uniform policy $\forwardalpha^i=\frac1p$ is then ergodic for $\lambda$, and Theorem~\ref{theo:universality_sphere} follows from master Theorem~\ref{theo:universality_true}.

\section{Quantitative Sampling Theorem}
\begin{theorem}
\label{theo:negative_control_appendix}
  Let $(\ForwardPolicyStar, \Foutstar)$ be a generative flow and let $\Finit$ and $\Fterm$ initial and terminal distributions. Assume that $\Finit\neq 0$, and consider the sample $s_\tau$ of the generative flow Markov chain from $\Finit$ to $\Ftermhat$ then:
    \begin{equation}
        \TV\left(s_\tau \left\| \frac{1}{\Ftermhat(\Statespace )} \Ftermhat\right.\right) \leq  \frac{\delta\Finit(\Statespace)}{\Ftermhat(\Statespace)}
    \end{equation}
    with $\TV(\mu\|\nu) := \frac12\|\mu-\nu\| = \sup_{A}|\mu(A)-\nu(A)|$ the total variation distance between probability measures, where $\|\mu-\nu\|:=(\mu-\nu)^+(\Statespace)+(\mu-\nu)^-(\Statespace)$ and $A$ ranges over the measurable subsets of $\Statespace$ (the bound is attained on a two-point example).
\end{theorem}
\begin{proof}
    For any generative flow $\mathbb F:=(\ForwardPolicy^*,\Fout^*)$, define $\ForwardPolicy^{\tau,\mathbb F}$ the kernel $x\mapsto (s_{\tau}|s_1=x)$  for the stopping condition $\Ftermhat$. The sampling distribution $s_\tau$ of $\mathbb F$ is then $\frac{1}{\Finit(\Statespace)}\Finit \ForwardPolicy^{\tau,\mathbb F}$. 
    Furthermore, we have $\delta\Fterm(\Statespace)-\delta\Finit(\Statespace)=\Finit(\Statespace)+\Fout^* \pi^*(\Statespace)-\Fout^*(\Statespace)-\Fterm(\Statespace)=\Finit(\Statespace)-\Fterm(\Statespace)$.
    In particular, $\Ftermhat(\Statespace)\geq \Finit(\Statespace)> 0$ so $\Ftermhat \neq 0$.
    
    We notice that the generative flow $\mathbb F':=(\ForwardPolicy^*,\Fout^*)$ satisfies the flow-matching constraint from $\Finithat$ to $\Ftermhat$  and that $\pi^{\tau,\mathbb F'} = \pi^{\tau,\mathbb F}$.  
    The sampler of $\mathbb F$ from $\Finit$ thus terminates almost surely, with $\mathbb E\tau \le (\Ftermhat(\Statespace)+\Foutstar(\Statespace))/\Finit(\Statespace)$, since $\Finit\le\Finithat$ and $\mathbb F'$ is flow matching from $\Finithat$ (Theorem \ref{theo:sampling_theorem}); in particular $s_\tau$ is well defined.
    If $\delta\Finit(\Statespace)=0$, then $\Finithat=\Finit$, $\mathbb F$ is flow matching from $\Finit$ to $\Ftermhat$ and the bound, $0$, is Theorem \ref{theo:sampling_theorem}; assume $\delta\Finit(\Statespace)>0$.
    Applying the sampling Theorem to $\mathbb F'$, we obtain 
    \begin{eqnarray}
        \Finithat\ForwardPolicy^{\tau,\mathbb F'} &=& \Ftermhat\\
        (\Finit+\delta\Finit)\ForwardPolicy^{\tau,\mathbb F'} &=& \Ftermhat\\
        \Finit\ForwardPolicy^{\tau,\mathbb F'}+\delta\Finit\ForwardPolicy^{\tau,\mathbb F'} &=& \Ftermhat \\
         \alpha\times \underbrace{\left(\frac{1}{\Finit(\Statespace)}\Finit \ForwardPolicy^{\tau,\mathbb F}\right)}_{s_\tau}+(1-\alpha) \times  \frac{1}{\delta\Finit(\Statespace)}\delta\Finit\ForwardPolicy^{\tau,\mathbb F'} &=& \frac{1}{\Ftermhat(\Statespace)}\Ftermhat
    \end{eqnarray}
    with $\alpha = \frac{\Finit(\Statespace)}{\Ftermhat(\Statespace)}=\frac{\Ftermhat(\Statespace)-\delta\Finit(\Statespace)}{\Ftermhat(\Statespace)}=1-\frac{\delta\Finit(\Statespace)}{\Ftermhat(\Statespace)}$. 
    The result follows.
\end{proof}

\section{Implementation considerations}
\subsection{Density filtering}\label{sec:filter}
Since EGFs give us access to a tractable density $\ftermhat$, we can estimate its mean $m$ and standard deviation $\sigma$ on the state space to filter out any unwanted outliers when sampling. By choosing a lower bound saturation value $\widehat{f}_{sat} = m - k\sigma$ where $k$ defines how strong the filter should be, we can then define a stricter sampling density $\tilde{f}_{\mathrm{term}}$ as:
\begin{equation}
\tilde{f}_{term} = 
    \begin{cases}
       \text{$\ftermhat$} & \text{if $\ftermhat$} \geq \ftermhat\\
       0 & \text{otherwise}\\
    \end{cases}       
\end{equation}

Since the new sampling density $\tilde{f}_{\mathrm{term}} = 0$ for all points of low true density, the next state of the Markov Chain is guaranteed not to be the sink state. This ensures the EGF only generates samples associated with high density. The factor $k$ can be chosen manually to obtain a more strict or lenient filter, or it can be automatically optimized by recalculating the negative log-likelihood on the training dataset to systematically evaluate all $k$ values.

\begin{figure}[H]
    \centering
    \includegraphics[width=0.5\linewidth,clip]{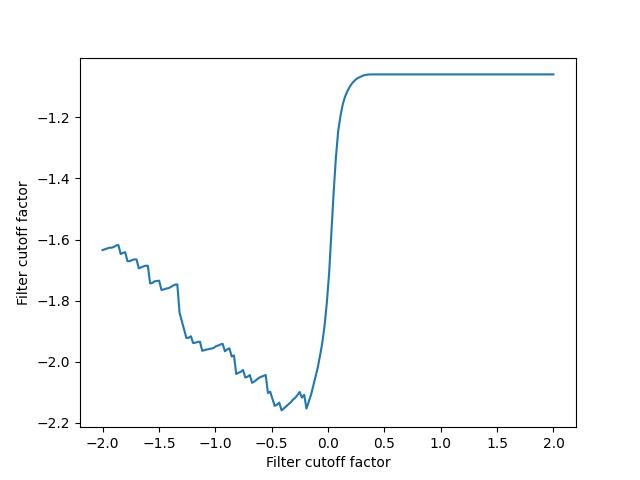}
    \caption{Negative log-likelihood (vertical axis) against the cutoff factor $k$ (horizontal axis; both axes are labelled `Filter cutoff factor') for a given model on the volcano dataset.}
\end{figure}

Note that this filtering is applied on the density rather than on the samples, meaning that no samples are filtered out. Therefore, this simply has the effect of preventing the trajectories from stopping at low density points. Figure~\ref{fig:filter_effect} demonstrates that using the optimal filter value concentrates the samples in the zones of high true density, bypassing the uniform background noise that would otherwise be found prior to filtering.
    
\begin{figure}[H]
    \centering

    \captionsetup[subfigure]{labelformat=simple}

    \begin{subfigure}{\linewidth}
        \centering
        \begin{subfigure}{0.33\linewidth}
            \centering
            \includegraphics[width=\linewidth,trim={14cm 6cm 14cm 7cm},clip]{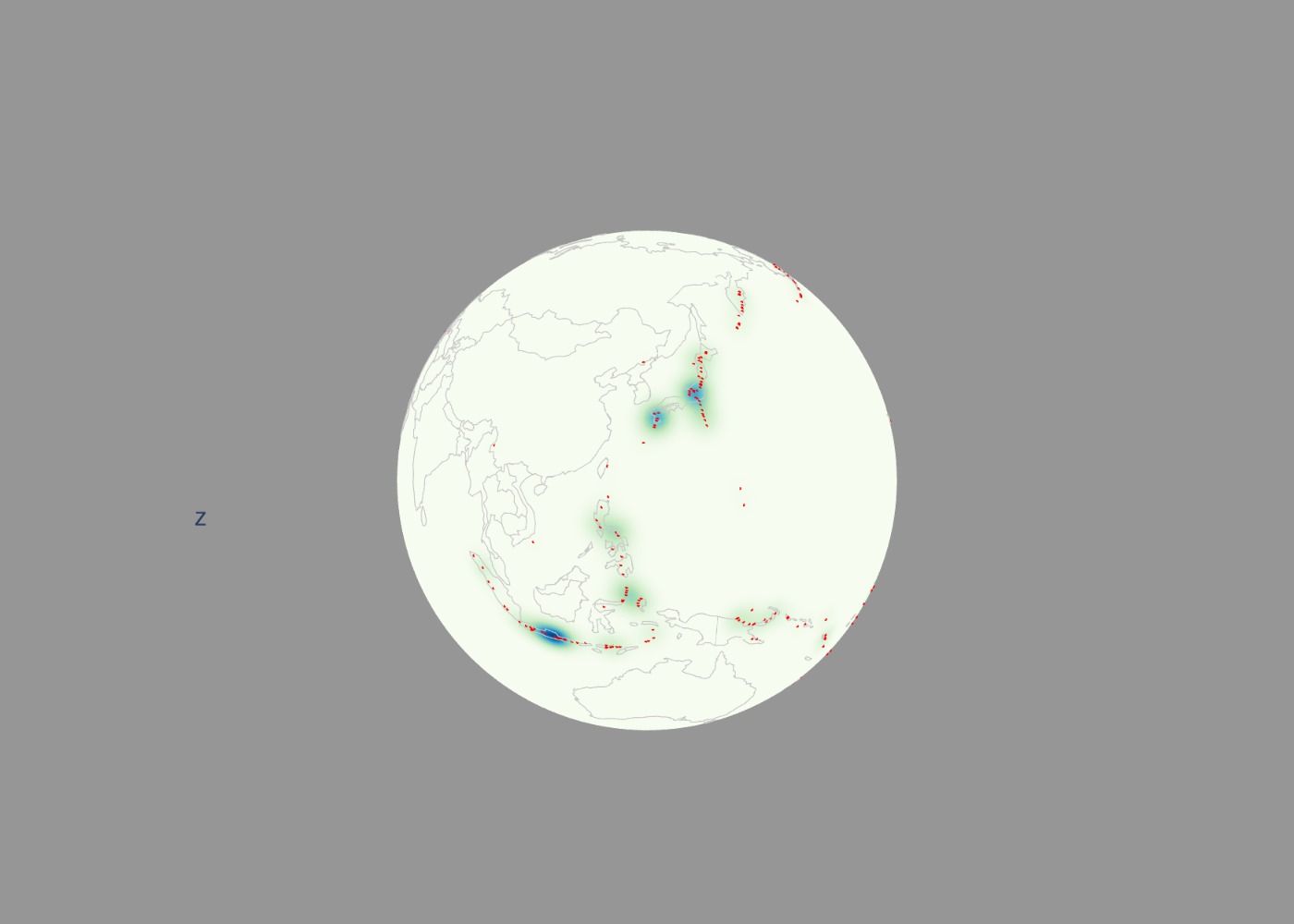}
            \caption{Target distribution}
        \end{subfigure}
        \hfill
        \begin{subfigure}{0.33\linewidth}
            \centering
            \includegraphics[width=\linewidth,trim={14cm 6cm 14cm 7cm},clip]{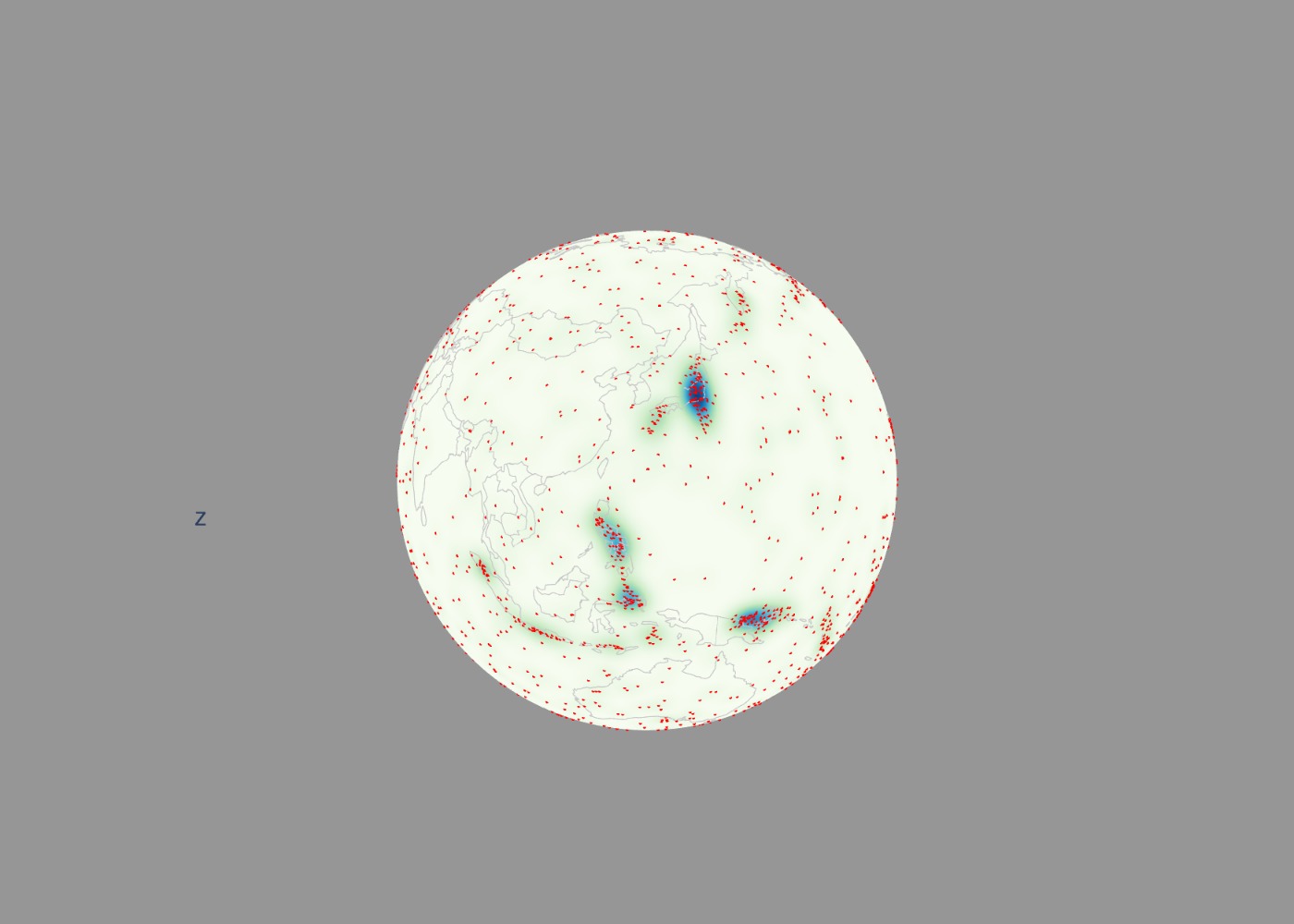}
            \caption{Unfiltered samples}
        \end{subfigure}
        \begin{subfigure}{0.33\linewidth}
            \centering
            \includegraphics[width=\linewidth,trim={14cm 6cm 14cm 7cm},clip]{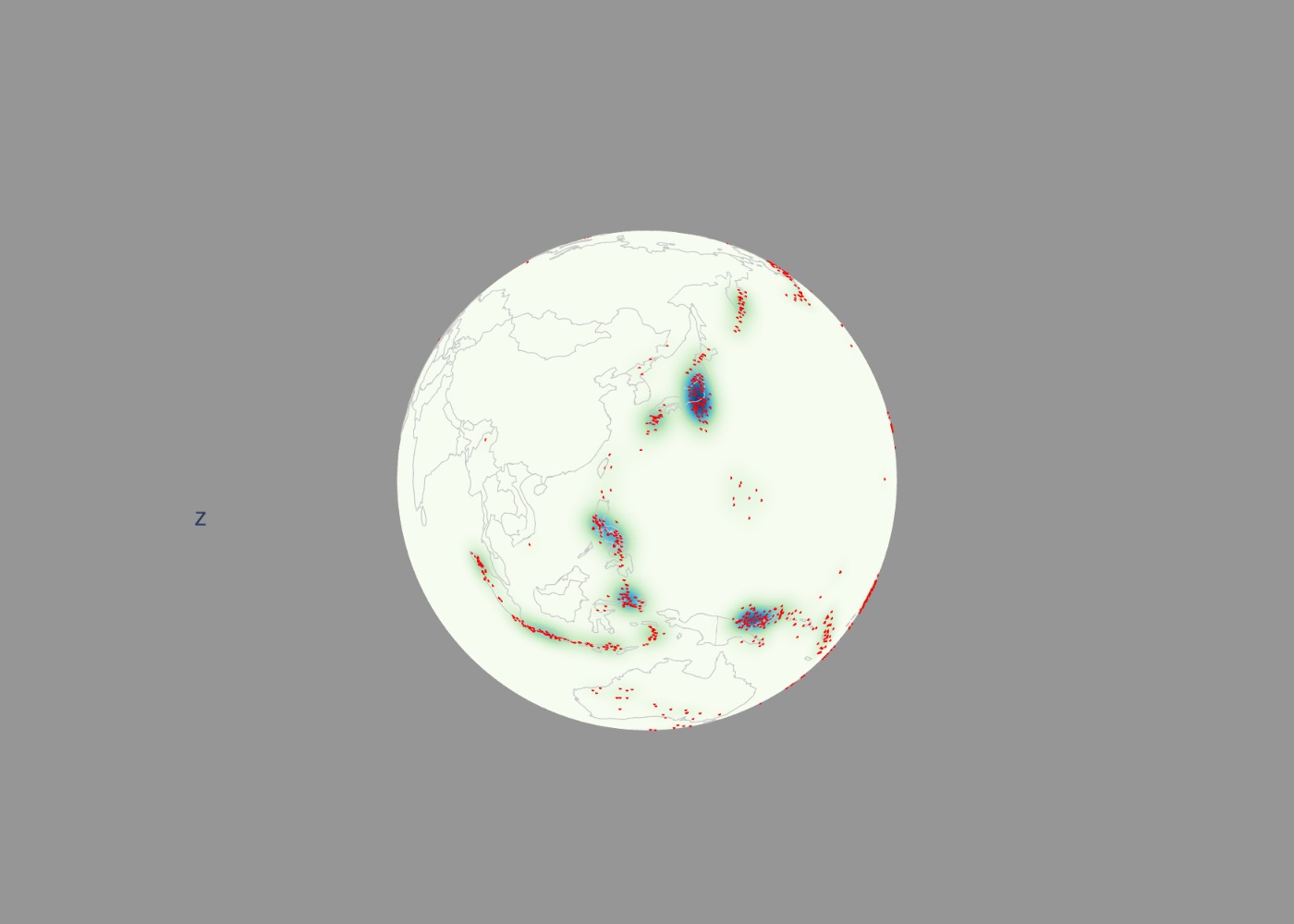}
            \caption{Samples with optimal filter}
        \end{subfigure}
        
    \end{subfigure}

    \caption{Comparison of the EGF's samples using the same model with and without a density filter.}
    \label{fig:filter_effect}
\end{figure}

\section{Ablation study on the size of EGF} \label{sec:size_ablation}


\begin{figure}
    \centering
    \captionsetup[subfigure]{labelformat=simple}
    \begin{subfigure}{\linewidth}
        \centering
        \begin{subfigure}{0.19\linewidth}
            \centering
            \includegraphics[width=\linewidth,trim={4cm 5cm 6cm 5cm},clip]{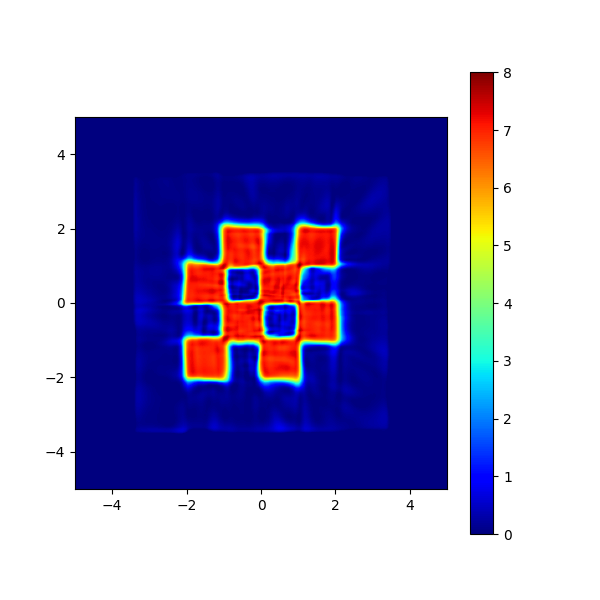}
        \end{subfigure}
        \hfill
        \begin{subfigure}{0.19\linewidth}
            \centering
            \includegraphics[width=\linewidth,trim={4cm 5cm 6cm 5cm},clip]{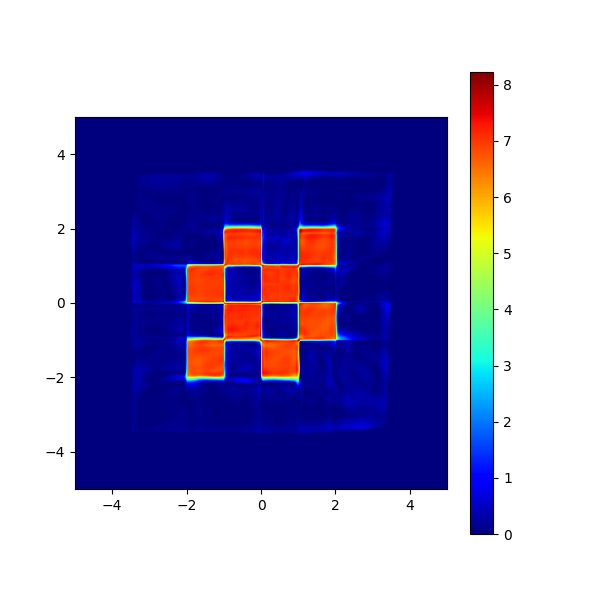}
        \end{subfigure}
        \hfill
        \begin{subfigure}{0.19\linewidth}
            \centering
            \includegraphics[width=\linewidth,trim={4cm 5cm 6cm 5cm},clip]{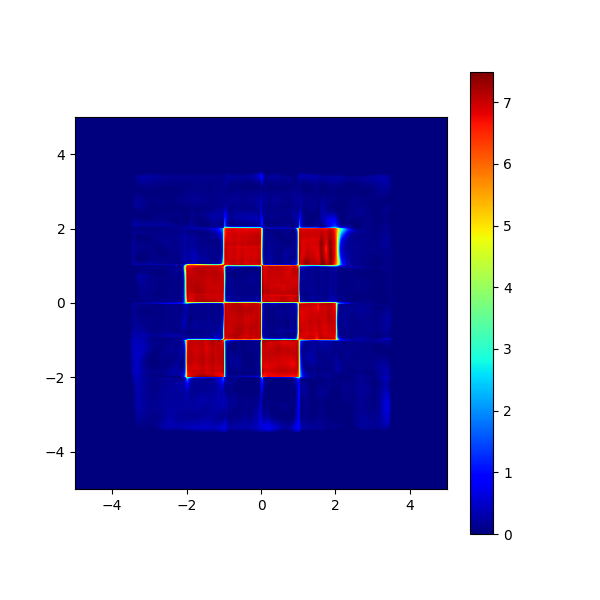}
        \end{subfigure}
        \hfill
        \begin{subfigure}{0.19\linewidth}
            \centering
            \includegraphics[width=\linewidth,trim={4cm 5cm 6cm 5cm},clip]{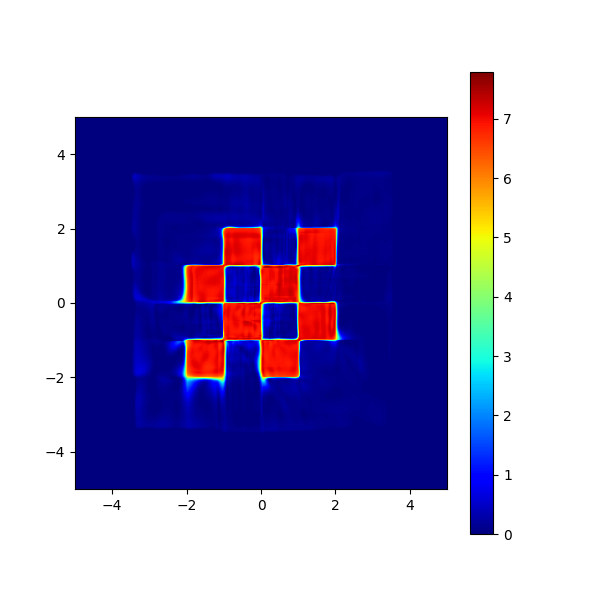}
        \end{subfigure}
        \hfill
        \begin{subfigure}{0.19\linewidth}
            \centering
            \includegraphics[width=\linewidth,trim={4cm 5cm 6cm 5cm},clip]{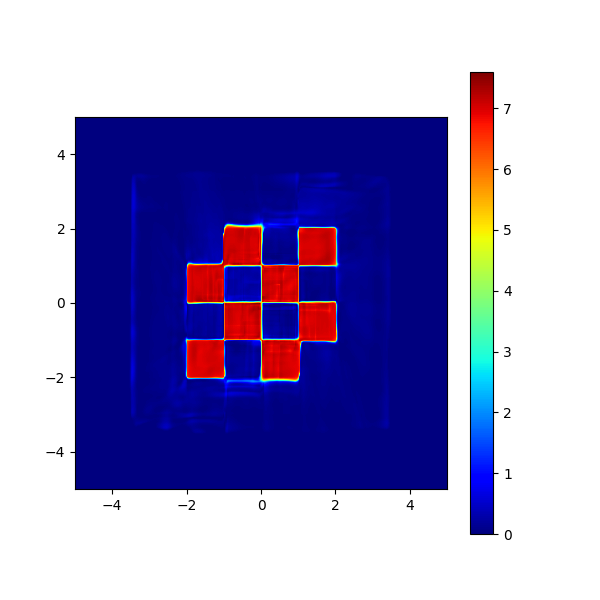}
        \end{subfigure}
        \caption{ From left to right, the depth of $\foutstar$ is kept constant while $\forwardalpha$ depth increases.}
    \end{subfigure}
    \begin{subfigure}{\linewidth}
        \centering
        \begin{subfigure}{0.19\linewidth}
            \centering
            \includegraphics[width=\linewidth,trim={4cm 5cm 6cm 5cm},clip]{ablation/22.png}
        \end{subfigure}
        \hfill
        \begin{subfigure}{0.19\linewidth}
            \centering
            \includegraphics[width=\linewidth,trim={4cm 5cm 6cm 5cm},clip]{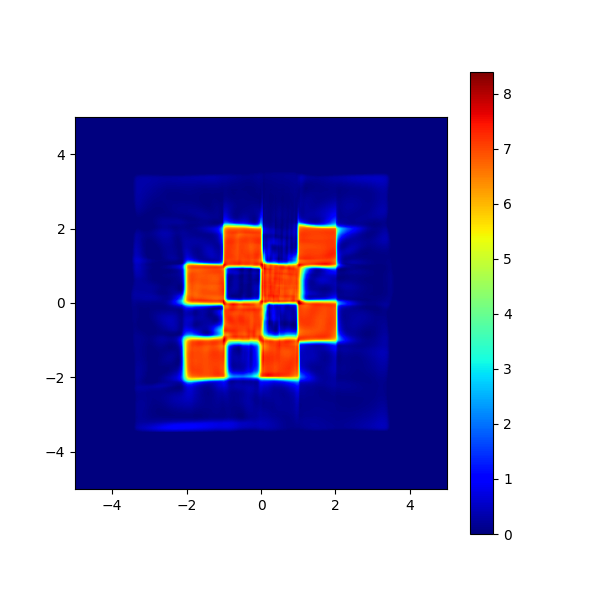}
        \end{subfigure}
        \hfill
        \begin{subfigure}{0.19\linewidth}
            \centering
            \includegraphics[width=\linewidth,trim={4cm 5cm 6cm 5cm},clip]{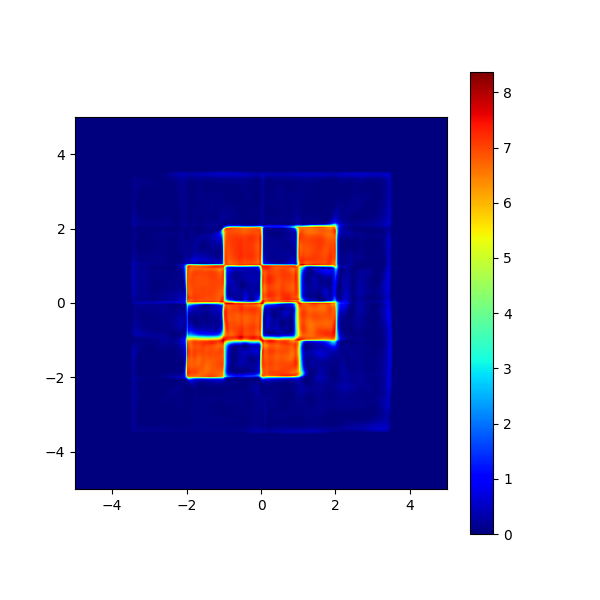}
        \end{subfigure}
        \hfill
        \begin{subfigure}{0.19\linewidth}
            \centering
            \includegraphics[width=\linewidth,trim={4cm 5cm 6cm 5cm},clip]{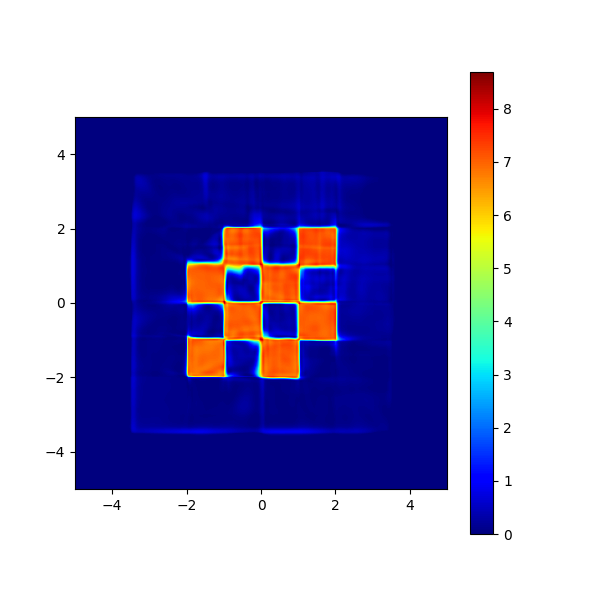}
        \end{subfigure}
        \hfill
        \begin{subfigure}{0.19\linewidth}
            \centering
            \includegraphics[width=\linewidth,trim={4cm 5cm 6cm 5cm},clip]{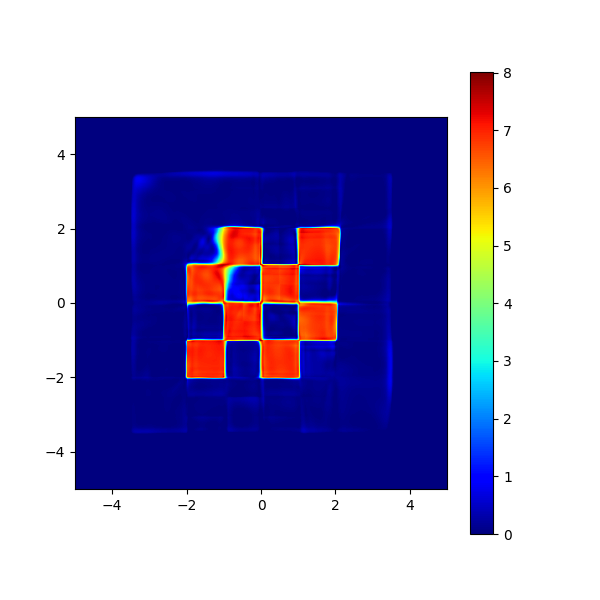}
        \end{subfigure}
        \caption{From left to right, the depth of $\forwardalpha$ is kept constant while $\foutstar$ depth increases.} 
    \end{subfigure}

    \caption{Variation of $\foutstar$ and $\forwardalpha$ depths from 2x32 to 6x32. The theoretical prediction is that $\forwardalpha$ may be kept constant while $\foutstar$ is trained. Keeping $\forwardalpha$ constant indeed yields improvements, however the improvement obtained by increasing the depth of $\forwardalpha$ instead of $\foutstar$ is comparatively better.}
    \label{fig:depth_ablation}
\end{figure}

\begin{figure}
    \centering
    \captionsetup[subfigure]{labelformat=simple}
    \begin{subfigure}{\linewidth}
        \centering
        \begin{subfigure}{0.48\linewidth}
            \centering
            \includegraphics[width=\linewidth]{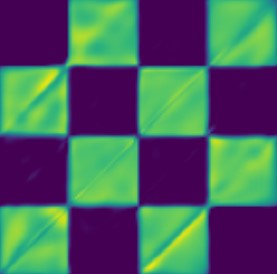}
        \end{subfigure}
        \hfill
        \begin{subfigure}{0.48\linewidth}
            \centering
            \includegraphics[width=\linewidth]{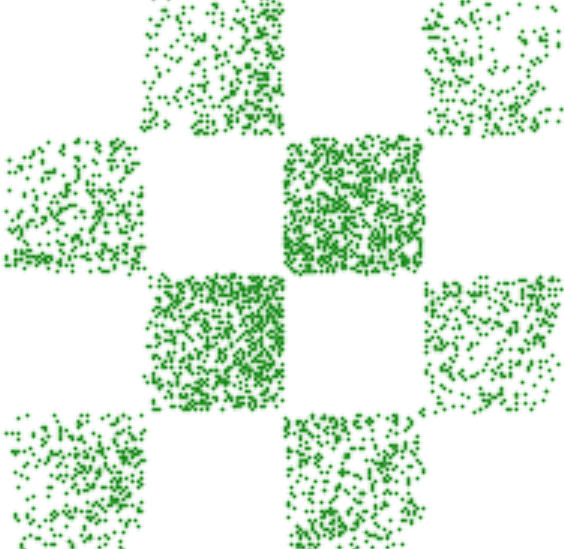}
        \end{subfigure}
    \end{subfigure}
    \caption{EGF with the minimal four affine moves $S^{\pm1},T^{\pm1}$ of Theorem~\ref{theo:universality_tori} on the torus $\mathbb T^2$, i.e.\ two generators of $\mathrm{SL}(2,\mathbb Z)$ together with their inverses. The MLPs parameterizing $\foutstar$ and $\forwardalpha$ have size 4x128. We see that the generated checkerboard has high quality, sharp boundaries but some under densities.}
    \label{fig:minimal_checker}
\end{figure}

\begin{figure}
    \centering
    \captionsetup[subfigure]{labelformat=simple}
    \begin{subfigure}{\linewidth}
        \centering
        \begin{subfigure}{0.24\linewidth}
            \centering
            \includegraphics[width=\linewidth,trim={3.5cm 4.5cm 5.5cm 4.5cm},clip]{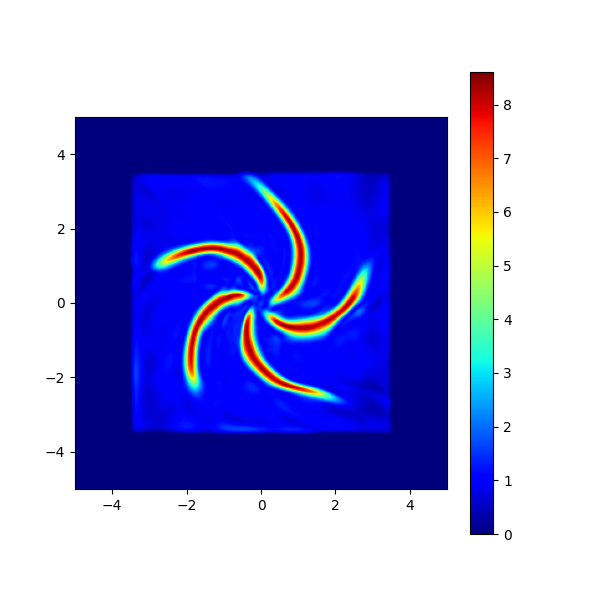}
        \end{subfigure}
        \hfill
        \begin{subfigure}{0.24\linewidth}
            \centering
            \includegraphics[width=\linewidth,trim={3.5cm 4.5cm 5.5cm 4.5cm},clip]{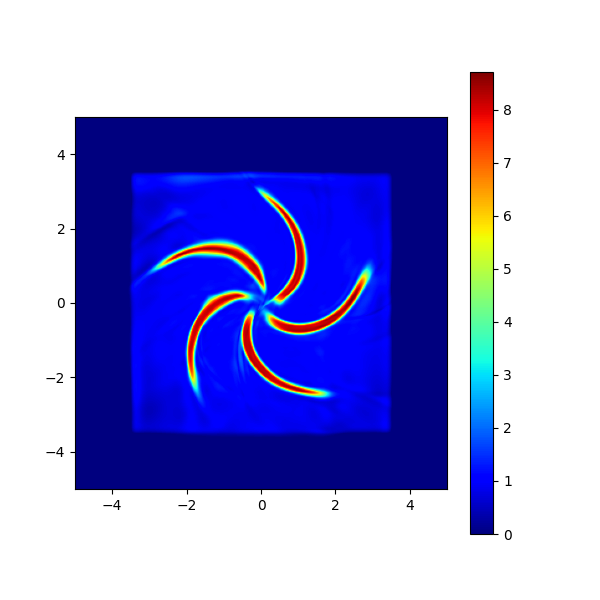}
        \end{subfigure}
        \hfill
        \begin{subfigure}{0.24\linewidth}
            \centering
            \includegraphics[width=\linewidth,trim={3.5cm 4.5cm 5.5cm 4.5cm},clip]{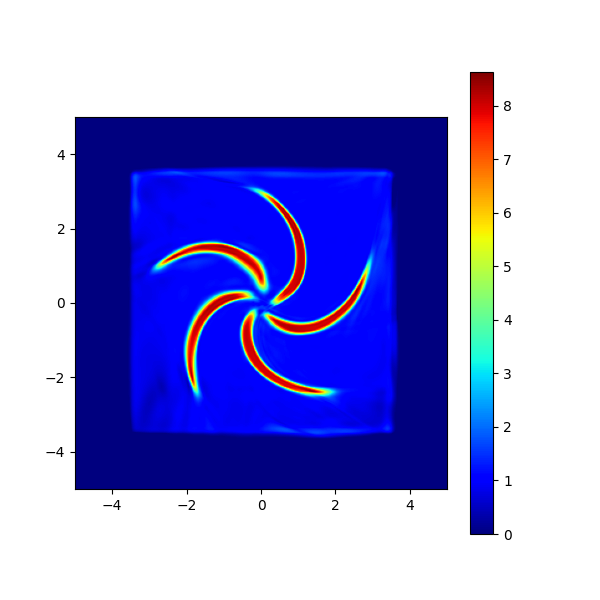}
        \end{subfigure}
        \hfill
        \begin{subfigure}{0.24\linewidth}
            \centering
            \includegraphics[width=\linewidth,trim={3.5cm 4.5cm 5.5cm 4.5cm},clip]{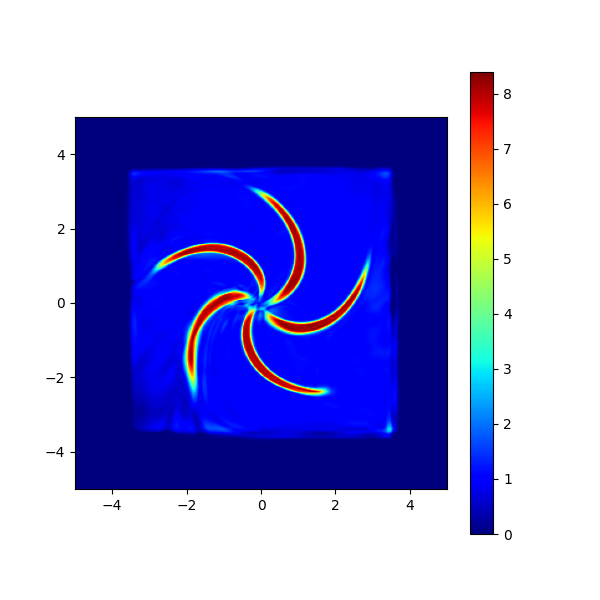}
        \end{subfigure}
        \hfill
        \begin{subfigure}{0.24\linewidth}
            \centering
            \includegraphics[width=\linewidth,trim={3.5cm 4.5cm 5.5cm 4.5cm},clip]{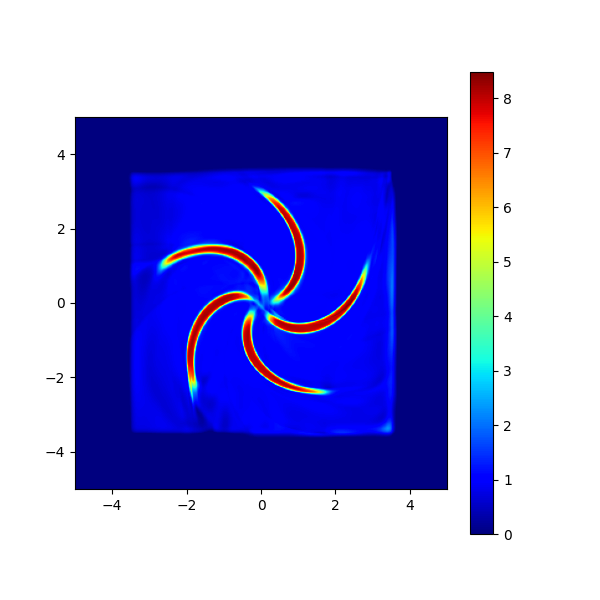}
        \end{subfigure}
        \hfill
        \begin{subfigure}{0.24\linewidth}
            \centering
            \includegraphics[width=\linewidth,trim={3.5cm 4.5cm 5.5cm 4.5cm},clip]{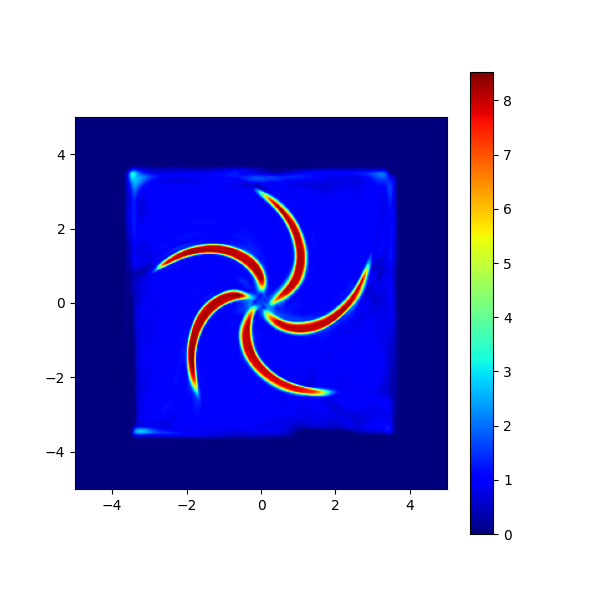}
        \end{subfigure}
        \hfill
        \begin{subfigure}{0.24\linewidth}
            \centering
            \includegraphics[width=\linewidth,trim={3.5cm 4.5cm 5.5cm 4.5cm},clip]{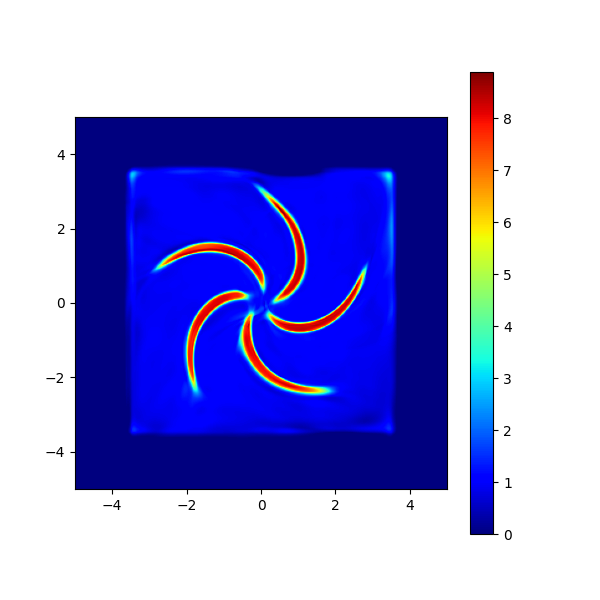}
        \end{subfigure}
        \hfill
        \begin{subfigure}{0.24\linewidth}
            \centering
            \includegraphics[width=\linewidth,trim={3.5cm 4.5cm 5.5cm 4.5cm},clip]{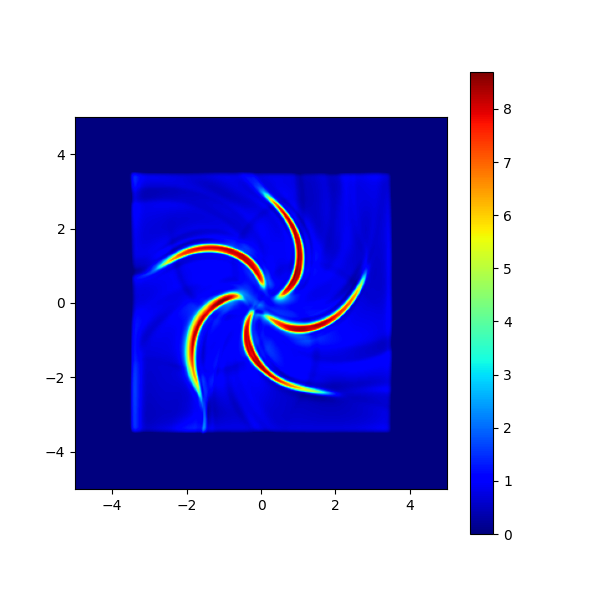}
        \end{subfigure}
        \caption{Generation of the spin wheel with variations of MLPs for both $\foutstar$ and $\forwardalpha$ from 2x32 to 5x32 (TOP) and 6x32 to 9x32 (BOT). }
    \end{subfigure}

      \begin{subfigure}{\linewidth}
        \centering
        \begin{subfigure}{0.24\linewidth}
            \centering
            \includegraphics[width=\linewidth,trim={3.5cm 4.5cm 5.5cm 4.5cm},clip]{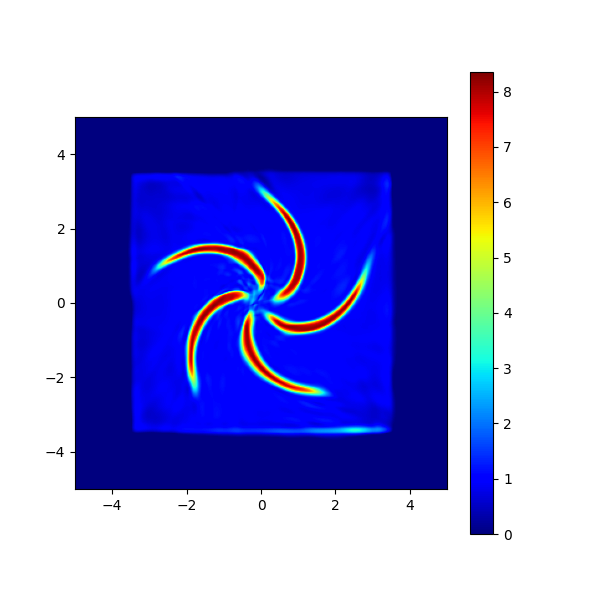}
        \end{subfigure}
        \hfill
       \begin{subfigure}{0.24\linewidth}
            \centering
            \includegraphics[width=\linewidth,trim={3.5cm 4.5cm 5.5cm 4.5cm},clip]{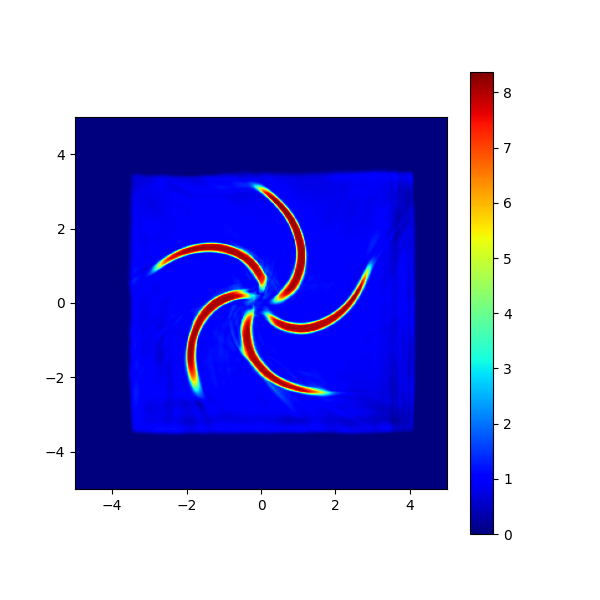}
        \end{subfigure}
        \hfill
        \begin{subfigure}{0.24\linewidth}
            \centering
            \includegraphics[width=\linewidth,trim={3.5cm 4.5cm 5.5cm 4.5cm},clip]{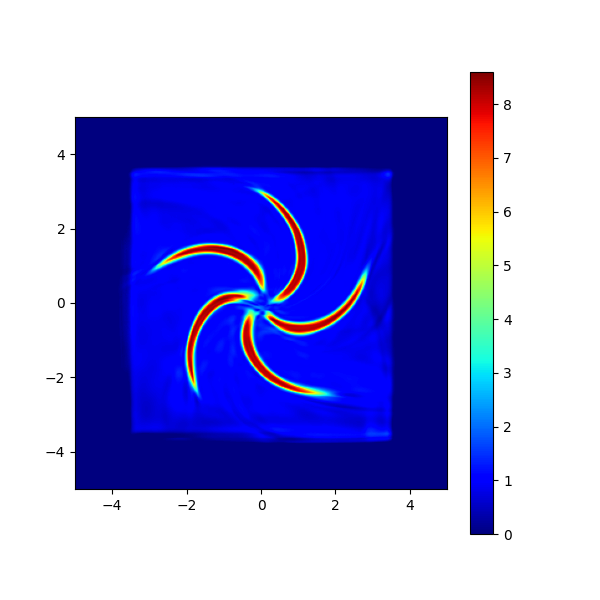}
        \end{subfigure}
        \hfill
        \begin{subfigure}{0.24\linewidth}
            \centering
            \includegraphics[width=\linewidth,trim={3.5cm 4.5cm 5.5cm 4.5cm},clip]{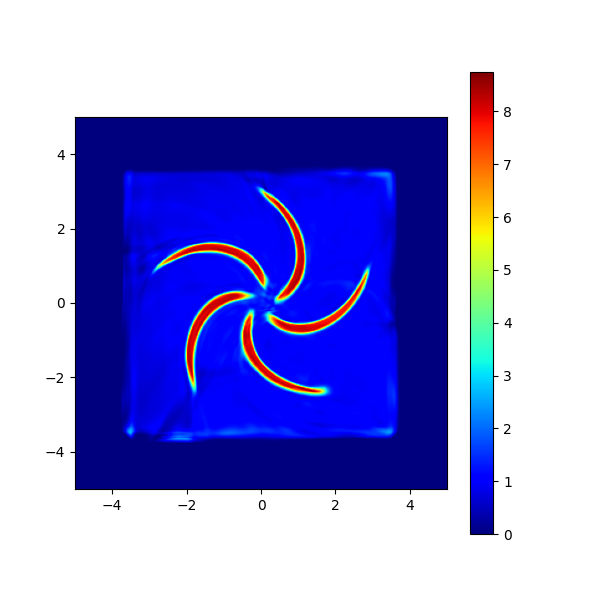}
        \end{subfigure}
        \hfill
        \begin{subfigure}{0.24\linewidth}
            \centering
            \includegraphics[width=\linewidth,trim={3.5cm 4.5cm 5.5cm 4.5cm},clip]{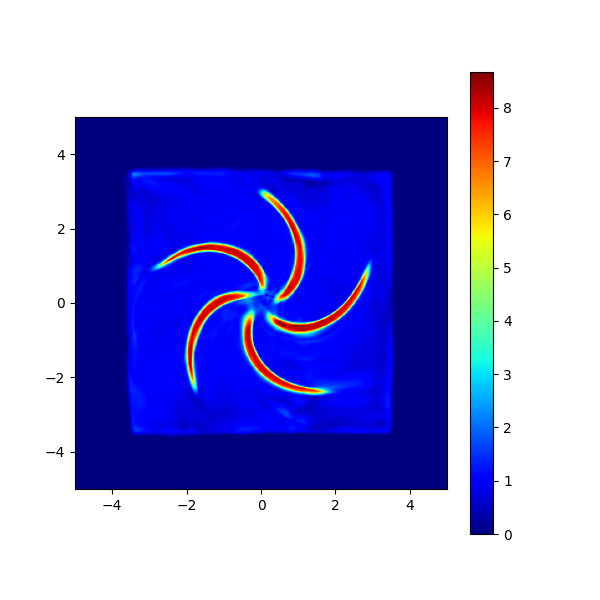}
        \end{subfigure}
        \hfill
        \begin{subfigure}{0.24\linewidth}
            \centering
            \includegraphics[width=\linewidth,trim={3.5cm 4.5cm 5.5cm 4.5cm},clip]{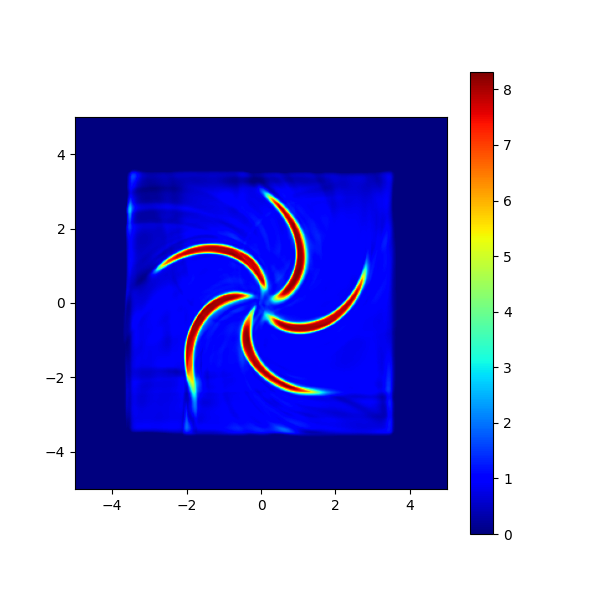}
        \end{subfigure}
        \hfill
        \begin{subfigure}{0.24\linewidth}
            \centering
            \includegraphics[width=\linewidth,trim={3.5cm 4.5cm 5.5cm 4.5cm},clip]{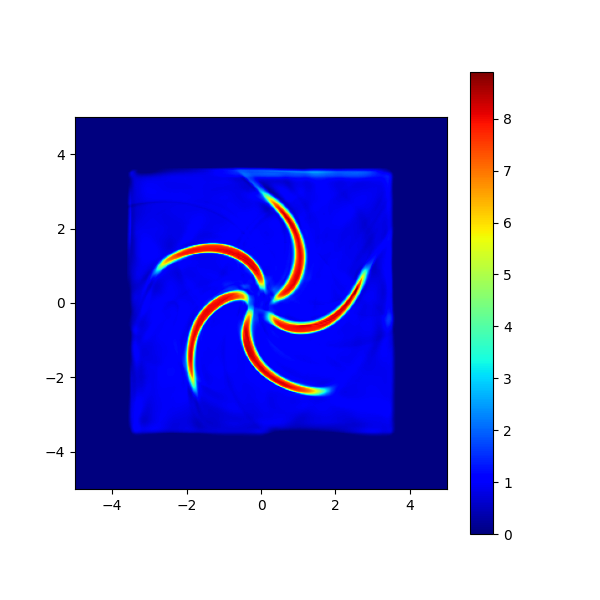}
        \end{subfigure}
        \hfill
        \begin{subfigure}{0.24\linewidth}
            \centering
            \includegraphics[width=\linewidth,trim={3.5cm 4.5cm 5.5cm 4.5cm},clip]{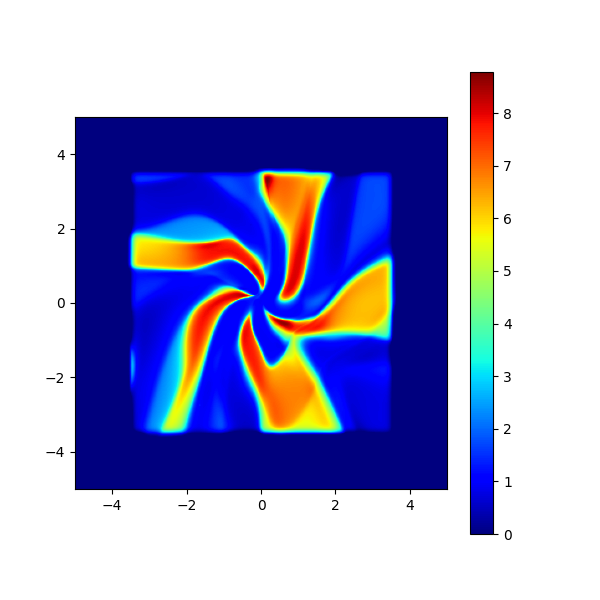}
        \end{subfigure}
        \hfill
       \caption{Generation of the spin wheel with variations of MLPs for both $\foutstar$ and $\forwardalpha$ from 2x64 to 5x64 (TOP) and 6x64 to 9x64 (BOT). }
    \end{subfigure}

    \caption{Variation of $\foutstar$ and $\forwardalpha$ depths and widths from 2x32 to 9x64. Increasing depth and width increases quality until depth 4 then the training seemingly becomes less stable as shown by the 9x64 experiment. Since we keep the learning rate constant at 1e-3, it is likely that the MLPs training becomes unstable.
    The state space is $\Statespace = \mathbb R^2$,  to ensure $\Fout(\Statespace)<+\infty$, we constrain $\foutstar$ to be zero outside $[-4,4]^2$. 
The EGF has 8 transformations which are translations, while the MLPs for $\forwardalpha$ and $\foutstar$ vary from 2x32 (2 hidden layers of width 32) to 9x64. Learning rate is kept at 0.001.
    }
    \label{fig:size_star}
\end{figure}

\newpage
\begin{figure}
    \centering
    \captionsetup[subfigure]{labelformat=simple}
    \begin{subfigure}{\linewidth}
        \centering
        \begin{subfigure}{0.24\linewidth}
            \centering
            \includegraphics[width=\linewidth,trim={3.5cm 4.5cm 5.5cm 4.5cm},clip]{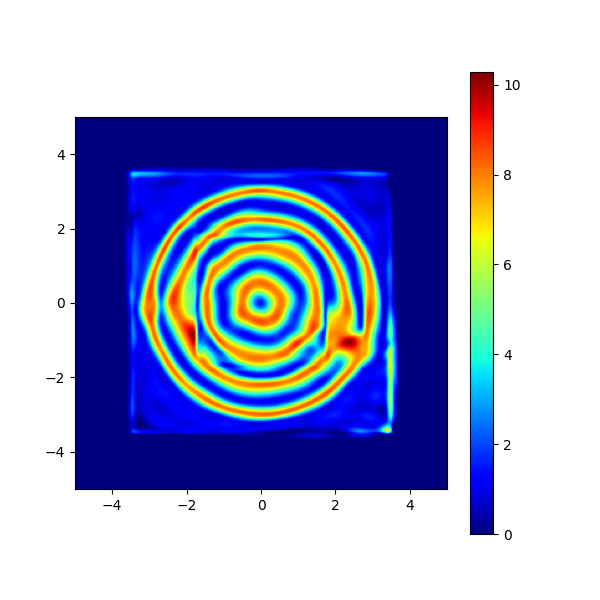}
        \end{subfigure}
        \hfill
        \begin{subfigure}{0.24\linewidth}
            \centering
            \includegraphics[width=\linewidth,trim={3.5cm 4.5cm 5.5cm 4.5cm},clip]{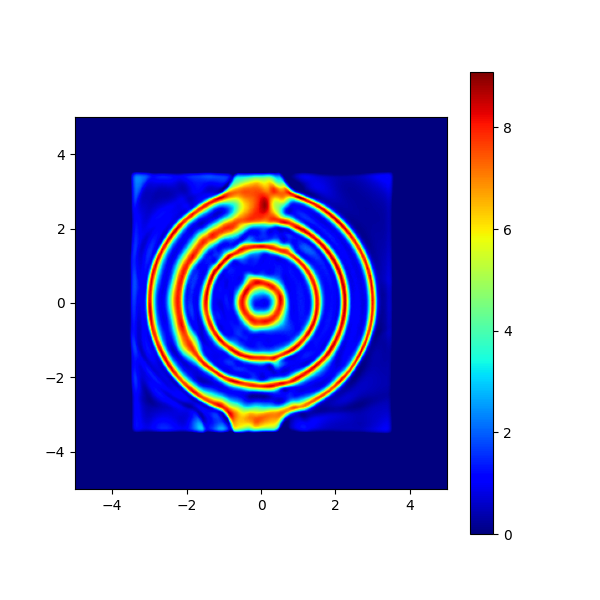}
        \end{subfigure}
        \hfill
        \begin{subfigure}{0.24\linewidth}
            \centering
            \includegraphics[width=\linewidth,trim={3.5cm 4.5cm 5.5cm 4.5cm},clip]{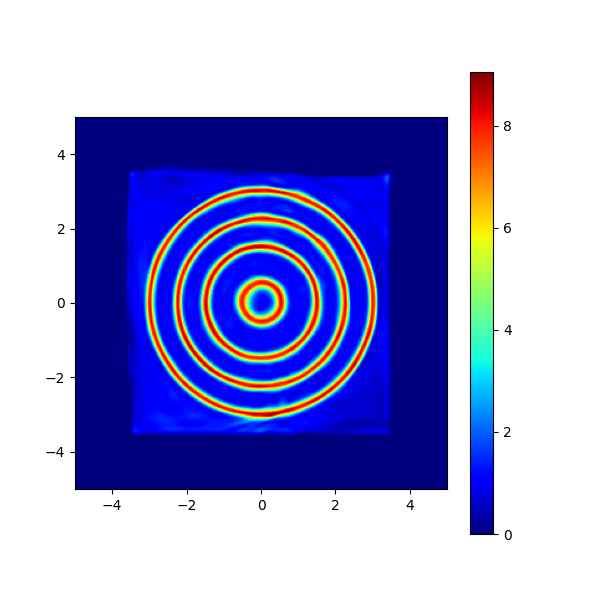}
        \end{subfigure}
        \hfill
        \begin{subfigure}{0.24\linewidth}
            \centering
            \includegraphics[width=\linewidth,trim={3.5cm 4.5cm 5.5cm 4.5cm},clip]{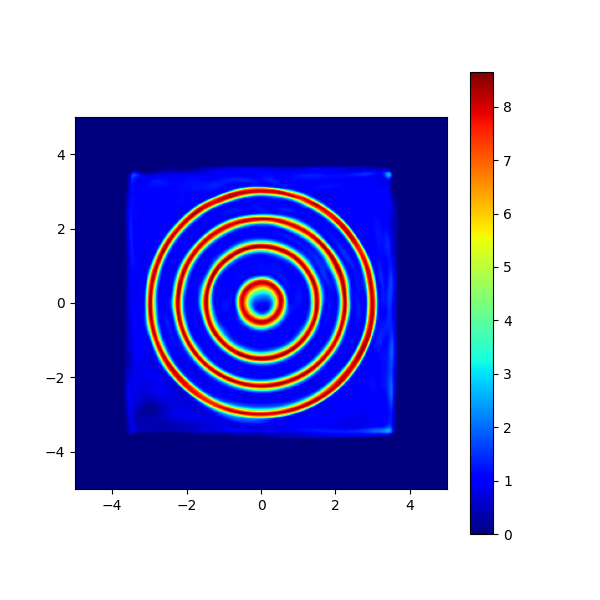}
        \end{subfigure}
        \hfill
        \begin{subfigure}{0.24\linewidth}
            \centering
            \includegraphics[width=\linewidth,trim={3.5cm 4.5cm 5.5cm 4.5cm},clip]{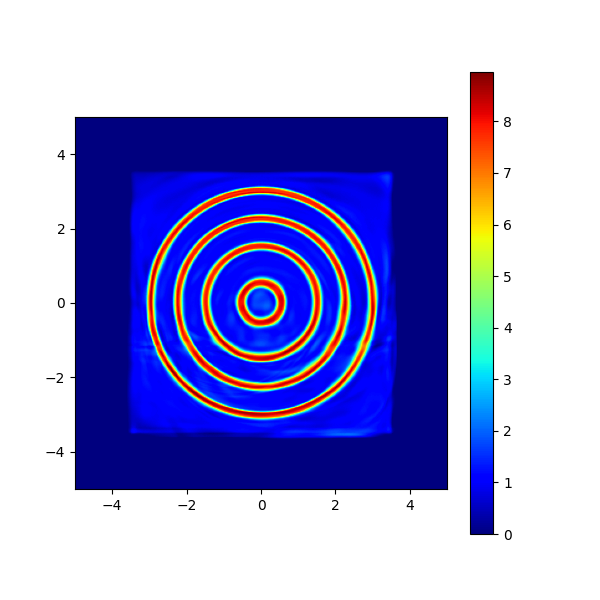}
        \end{subfigure}
        \hfill
        \begin{subfigure}{0.24\linewidth}
            \centering
            \includegraphics[width=\linewidth,trim={3.5cm 4.5cm 5.5cm 4.5cm},clip]{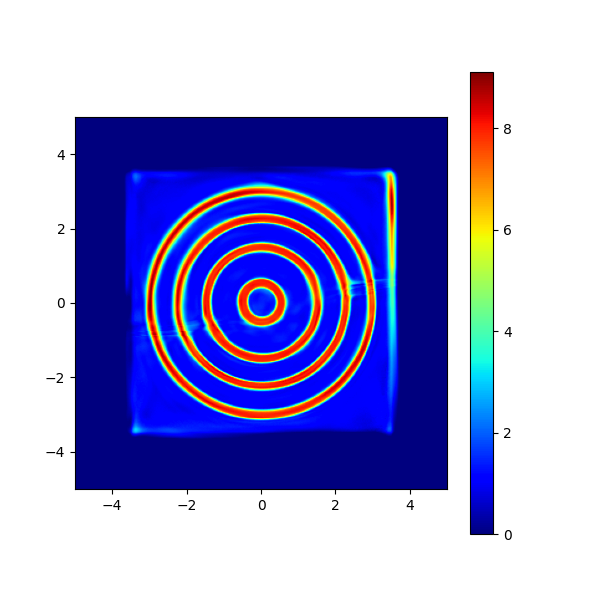}
        \end{subfigure}
        \hfill
        \begin{subfigure}{0.24\linewidth}
            \centering
            \includegraphics[width=\linewidth,trim={3.5cm 4.5cm 5.5cm 4.5cm},clip]{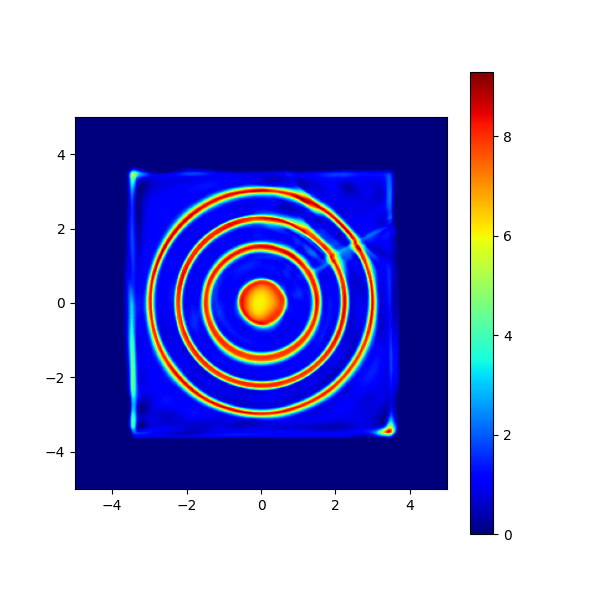}
        \end{subfigure}
        \hfill
        \begin{subfigure}{0.24\linewidth}
            \centering
            \includegraphics[width=\linewidth,trim={3.5cm 4.5cm 5.5cm 4.5cm},clip]{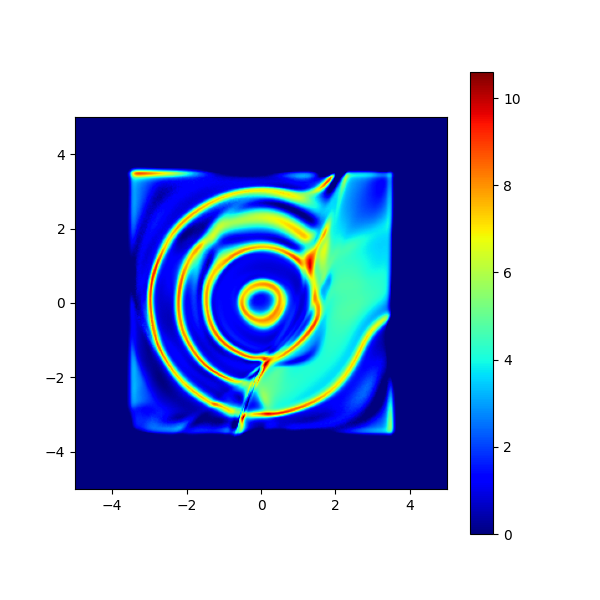}
        \end{subfigure}
        \caption{Generation of four circles with variations of MLPs for both $\foutstar$ and $\forwardalpha$ from 2x32 to 5x32 (TOP) and 6x32 to 9x32 (BOT). }
    \end{subfigure}

      \begin{subfigure}{\linewidth}
        \centering
        \begin{subfigure}{0.24\linewidth}
            \centering
            \includegraphics[width=\linewidth,trim={3.5cm 4.5cm 5.5cm 4.5cm},clip]{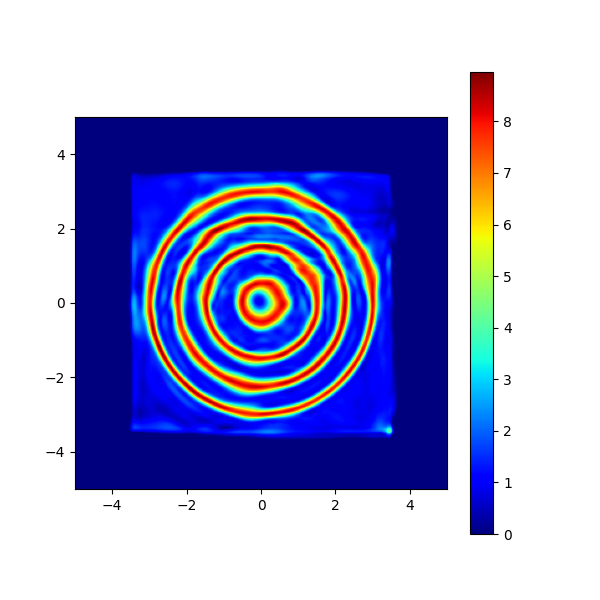}
        \end{subfigure}
        \hfill
       \begin{subfigure}{0.24\linewidth}
            \centering
            \includegraphics[width=\linewidth,trim={3.5cm 4.5cm 5.5cm 4.5cm},clip]{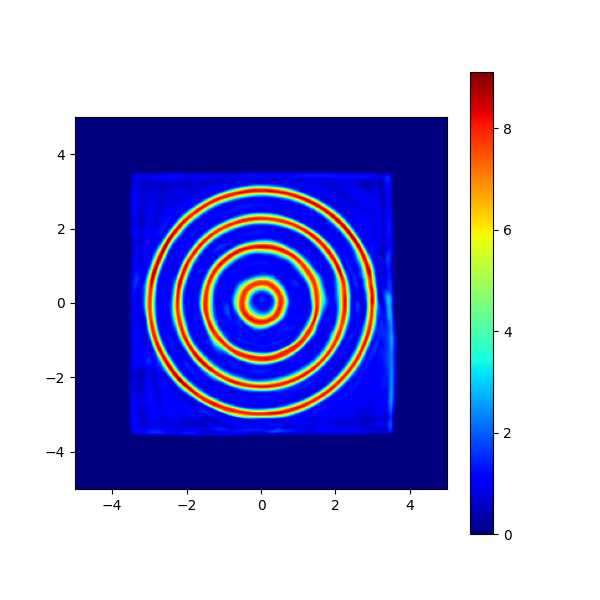}
        \end{subfigure}
        \hfill
        \begin{subfigure}{0.24\linewidth}
            \centering
            \includegraphics[width=\linewidth,trim={3.5cm 4.5cm 5.5cm 4.5cm},clip]{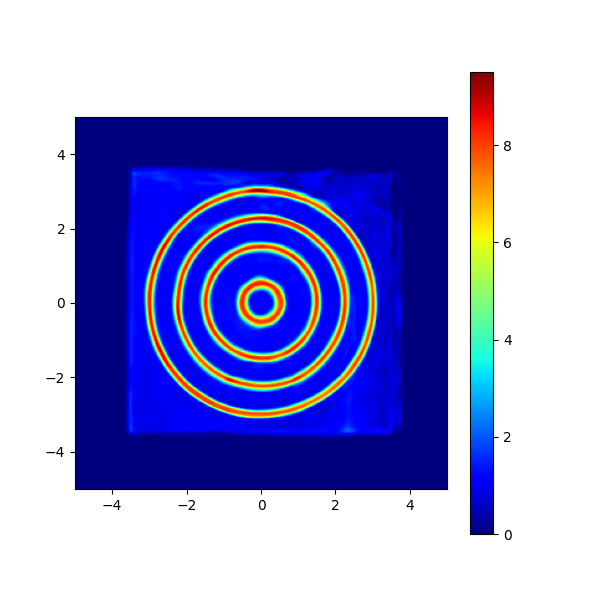}
        \end{subfigure}
        \hfill
        \begin{subfigure}{0.24\linewidth}
            \centering
            \includegraphics[width=\linewidth,trim={3.5cm 4.5cm 5.5cm 4.5cm},clip]{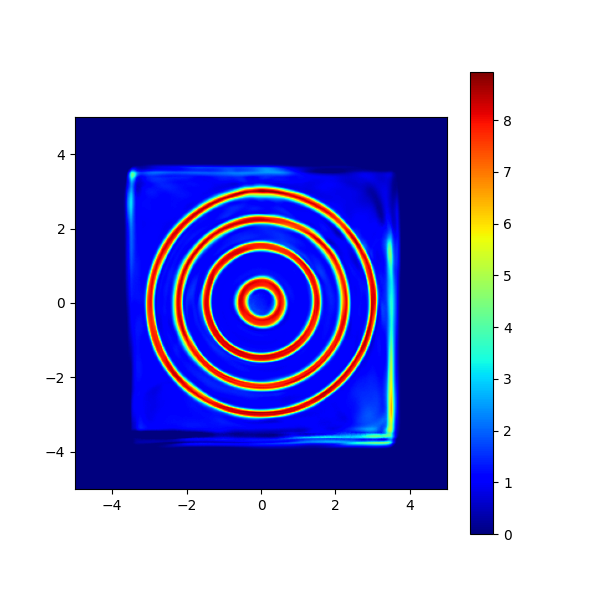}
        \end{subfigure}
        \hfill
        \begin{subfigure}{0.24\linewidth}
            \centering
            \includegraphics[width=\linewidth,trim={3.5cm 4.5cm 5.5cm 4.5cm},clip]{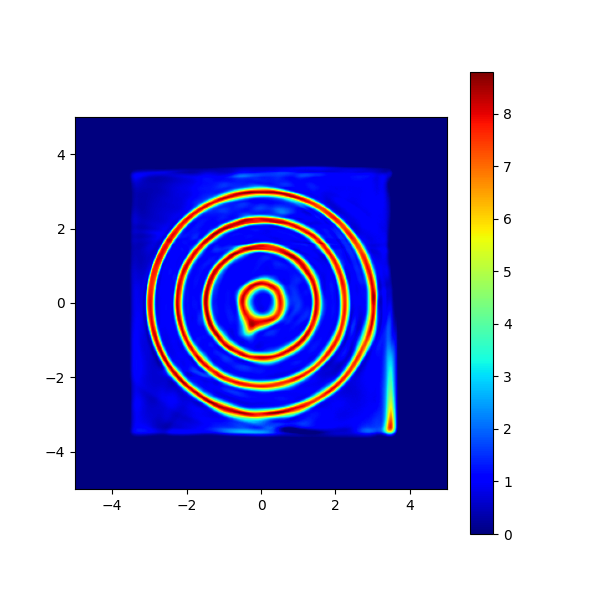}
        \end{subfigure}
        \hfill
        \begin{subfigure}{0.24\linewidth}
            \centering
            \includegraphics[width=\linewidth,trim={3.5cm 4.5cm 5.5cm 4.5cm},clip]{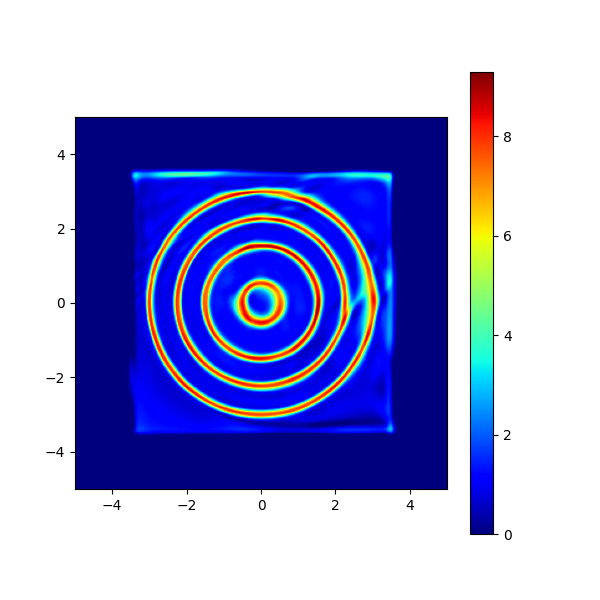}
        \end{subfigure}
        \hfill
        \begin{subfigure}{0.24\linewidth}
            \centering
            \includegraphics[width=\linewidth,trim={3.5cm 4.5cm 5.5cm 4.5cm},clip]{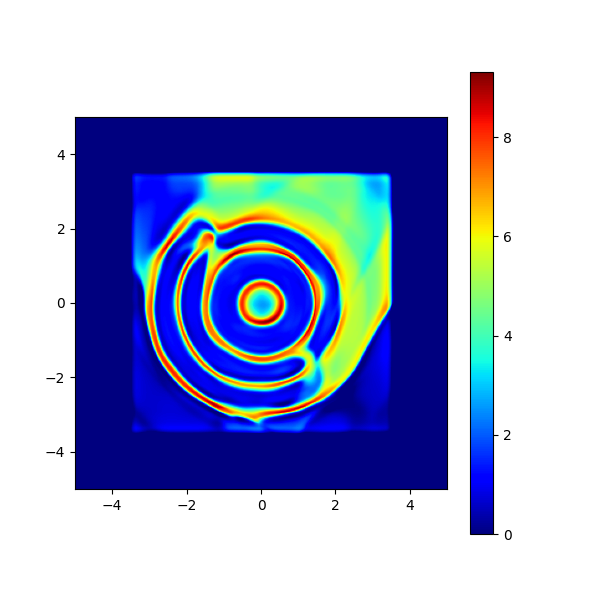}
        \end{subfigure}
        \hfill
        \begin{subfigure}{0.24\linewidth}
            \centering
            \includegraphics[width=\linewidth,trim={3.5cm 4.5cm 5.5cm 4.5cm},clip]{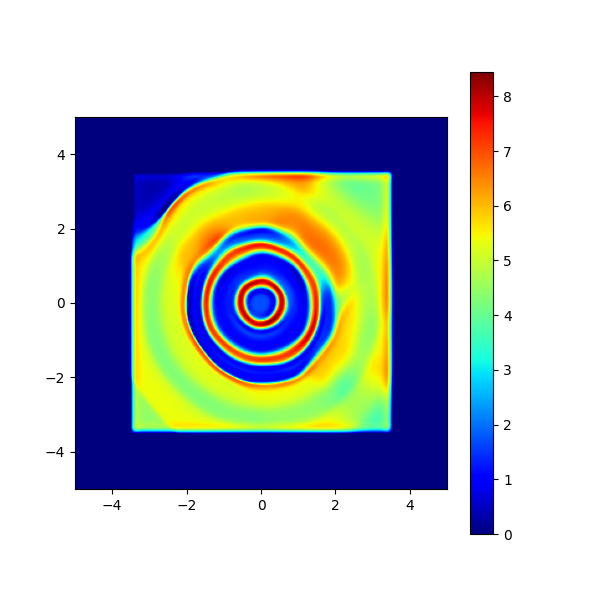}
        \end{subfigure}
        \hfill
       \caption{Generation of four circles with variations of MLPs for both $\foutstar$ and $\forwardalpha$ from 2x64 to 5x64 (TOP) and 6x64 to 9x64 (BOT). }
    \end{subfigure}

    \caption{Variation of $\foutstar$ and $\forwardalpha$ depths and widths from 2x32 to 9x64. Increasing depth and width increases quality until depth 4 then the training seemingly becomes less stable as shown by the 9x64 experiment. Since we keep the learning rate constant at 1e-3, it is likely that the MLPs training becomes unstable.
    The state space is $\Statespace = \mathbb R^2$,  to ensure $\Fout(\Statespace)<+\infty$, we constrain $\foutstar$ to be zero outside $[-4,4]^2$. 
The EGF has 8 transformations which are translations, while the MLPs for $\forwardalpha$ and $\foutstar$ vary from 2x32 (2 hidden layers of width 32) to 9x64. Learning rate is kept at 0.001.
    }
    \label{fig:size_circles}
\end{figure}

\begin{figure}
    \centering
    \captionsetup[subfigure]{labelformat=simple}
    \begin{subfigure}{\linewidth}
        \centering
        \begin{subfigure}{0.24\linewidth}
            \centering
            \includegraphics[width=\linewidth,trim={3.5cm 4.5cm 5.5cm 4.5cm},clip]{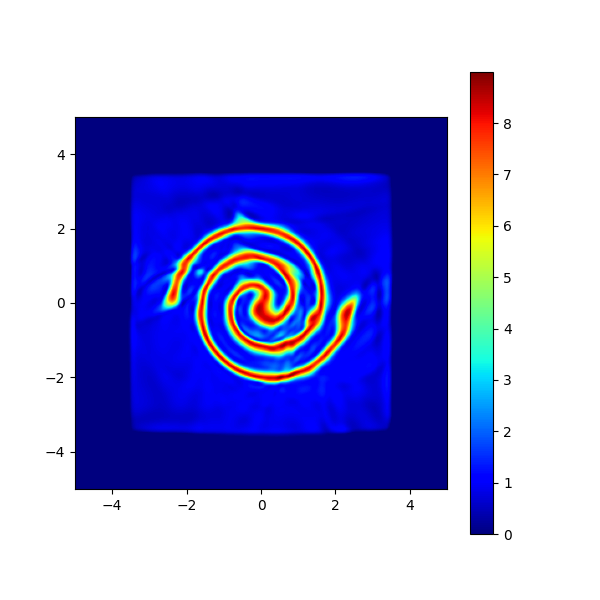}
        \end{subfigure}
        \hfill
        \begin{subfigure}{0.24\linewidth}
            \centering
            \includegraphics[width=\linewidth,trim={3.5cm 4.5cm 5.5cm 4.5cm},clip]{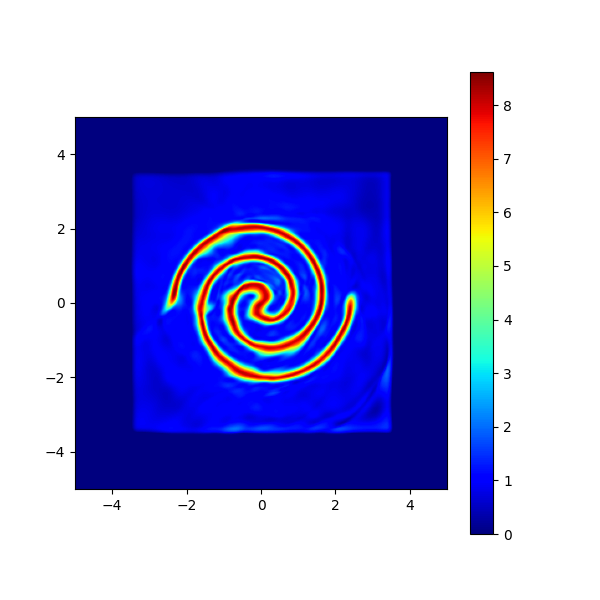}
        \end{subfigure}
        \hfill
        \begin{subfigure}{0.24\linewidth}
            \centering
            \includegraphics[width=\linewidth,trim={3.5cm 4.5cm 5.5cm 4.5cm},clip]{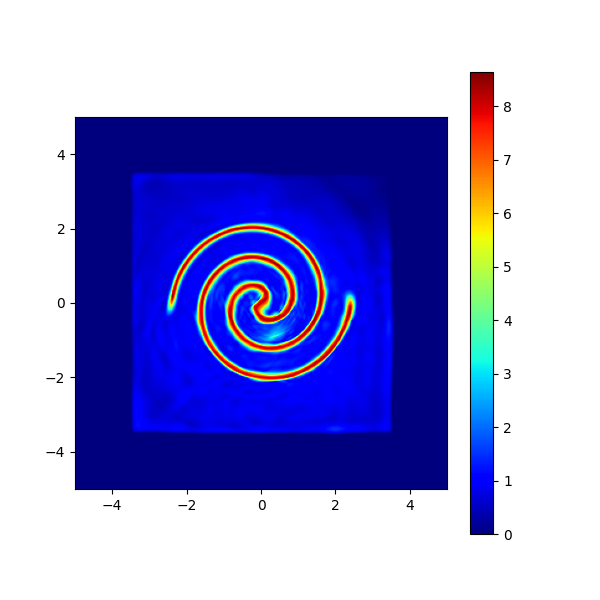}
        \end{subfigure}
        \hfill
        \begin{subfigure}{0.24\linewidth}
            \centering
            \includegraphics[width=\linewidth,trim={3.5cm 4.5cm 5.5cm 4.5cm},clip]{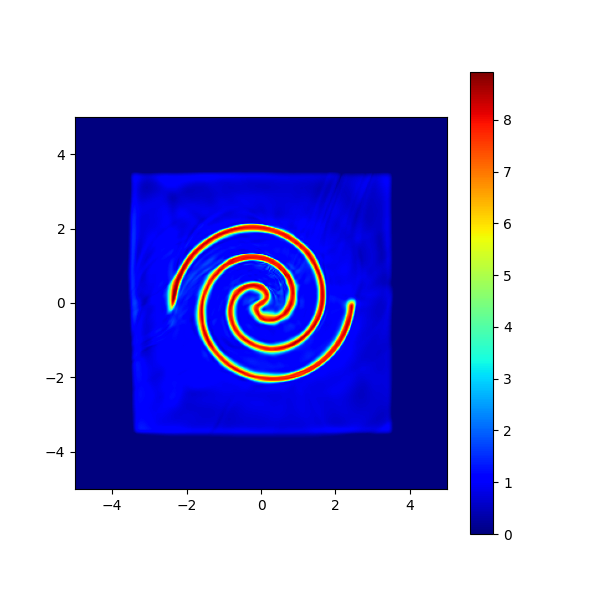}
        \end{subfigure}
        \hfill
        \begin{subfigure}{0.24\linewidth}
            \centering
            \includegraphics[width=\linewidth,trim={3.5cm 4.5cm 5.5cm 4.5cm},clip]{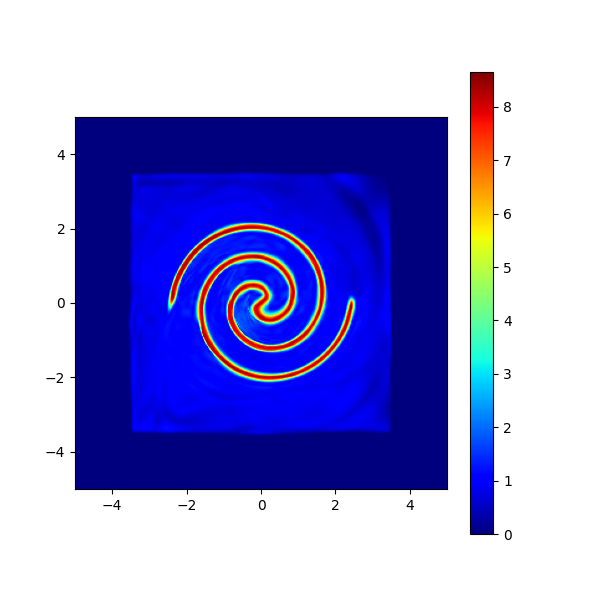}
        \end{subfigure}
        \hfill
        \begin{subfigure}{0.24\linewidth}
            \centering
            \includegraphics[width=\linewidth,trim={3.5cm 4.5cm 5.5cm 4.5cm},clip]{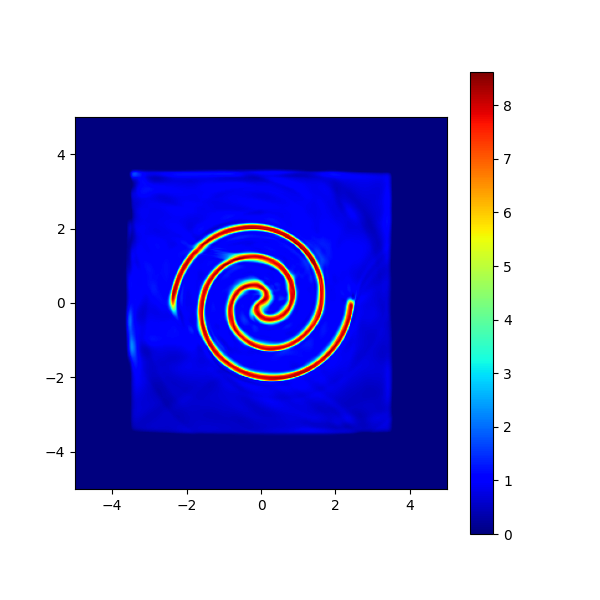}
        \end{subfigure}
        \hfill
        \begin{subfigure}{0.24\linewidth}
            \centering
            \includegraphics[width=\linewidth,trim={3.5cm 4.5cm 5.5cm 4.5cm},clip]{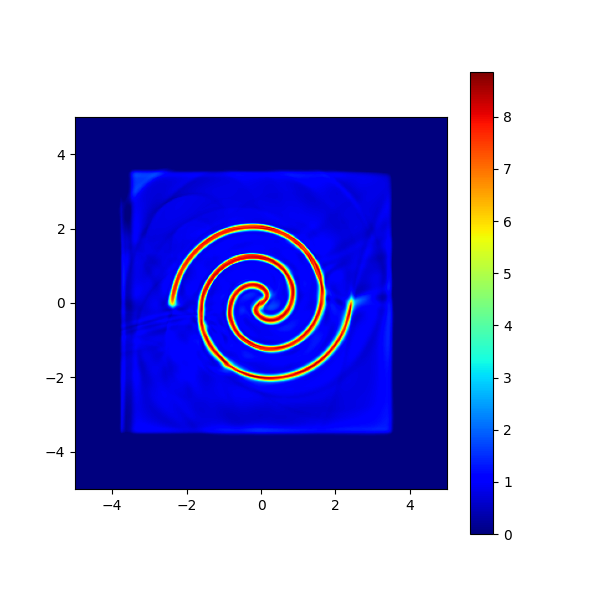}
        \end{subfigure}
        \hfill
        \begin{subfigure}{0.24\linewidth}
            \centering
            \includegraphics[width=\linewidth,trim={3.5cm 4.5cm 5.5cm 4.5cm},clip]{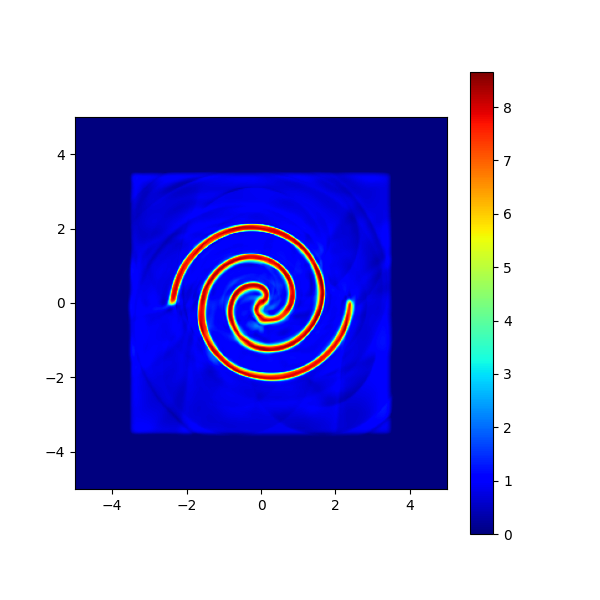}
        \end{subfigure}
        \caption{Generation of two spirals with variations of MLPs for both $\foutstar$ and $\forwardalpha$ from 2x32 to 5x32 (TOP) and 6x32 to 9x32 (BOT). }
    \end{subfigure}

      \begin{subfigure}{\linewidth}
        \centering
        \begin{subfigure}{0.24\linewidth}
            \centering
            \includegraphics[width=\linewidth,trim={3.5cm 4.5cm 5.5cm 4.5cm},clip]{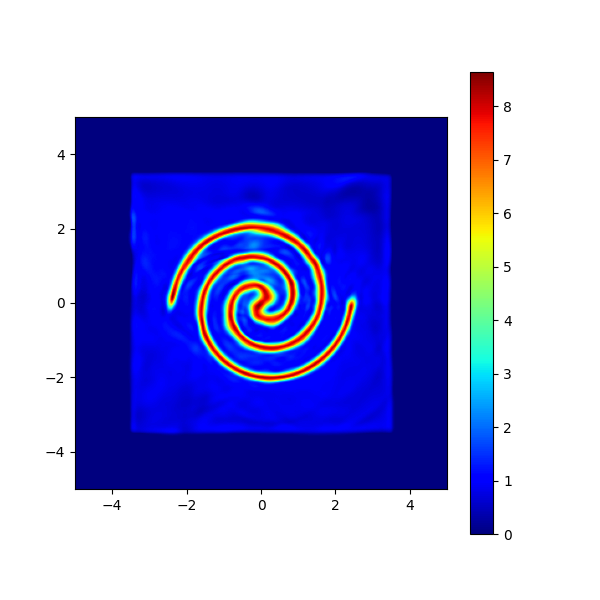}
        \end{subfigure}
        \hfill
       \begin{subfigure}{0.24\linewidth}
            \centering
            \includegraphics[width=\linewidth,trim={3.5cm 4.5cm 5.5cm 4.5cm},clip]{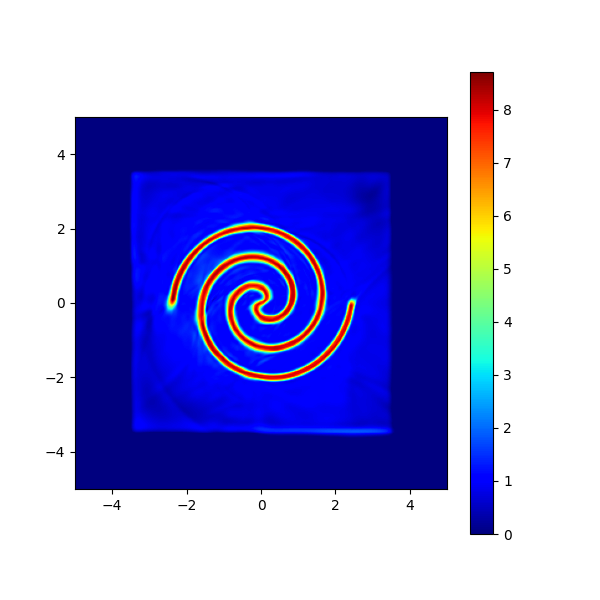}
        \end{subfigure}
        \hfill
        \begin{subfigure}{0.24\linewidth}
            \centering
            \includegraphics[width=\linewidth,trim={3.5cm 4.5cm 5.5cm 4.5cm},clip]{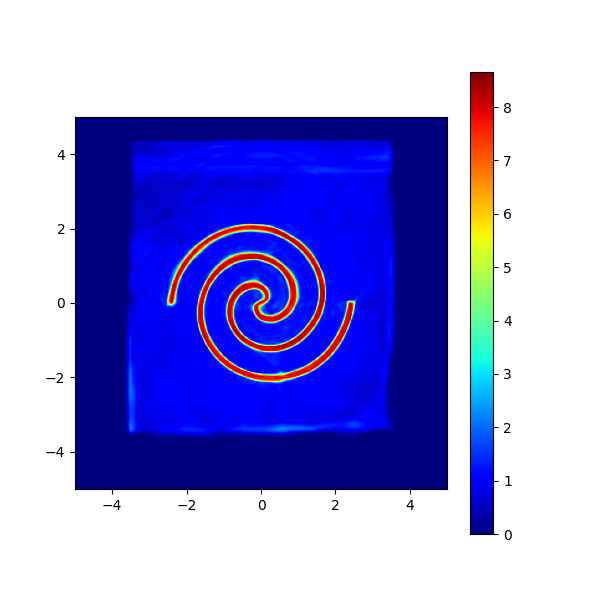}
        \end{subfigure}
        \hfill
        \begin{subfigure}{0.24\linewidth}
            \centering
            \includegraphics[width=\linewidth,trim={3.5cm 4.5cm 5.5cm 4.5cm},clip]{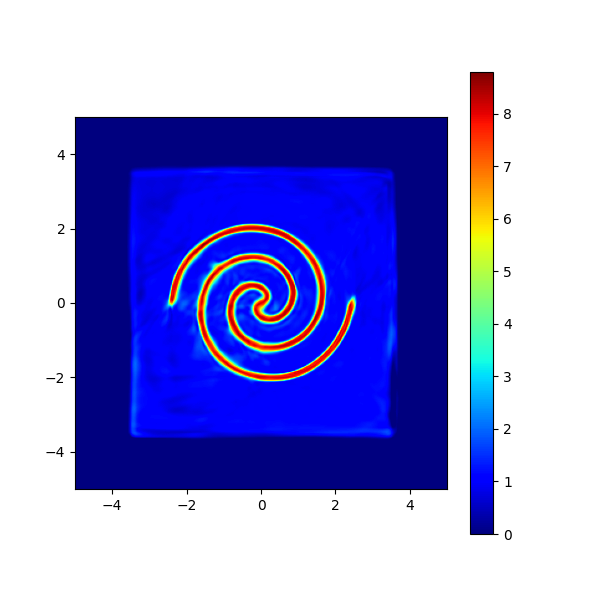}
        \end{subfigure}
        \hfill
        \begin{subfigure}{0.24\linewidth}
            \centering
            \includegraphics[width=\linewidth,trim={3.5cm 4.5cm 5.5cm 4.5cm},clip]{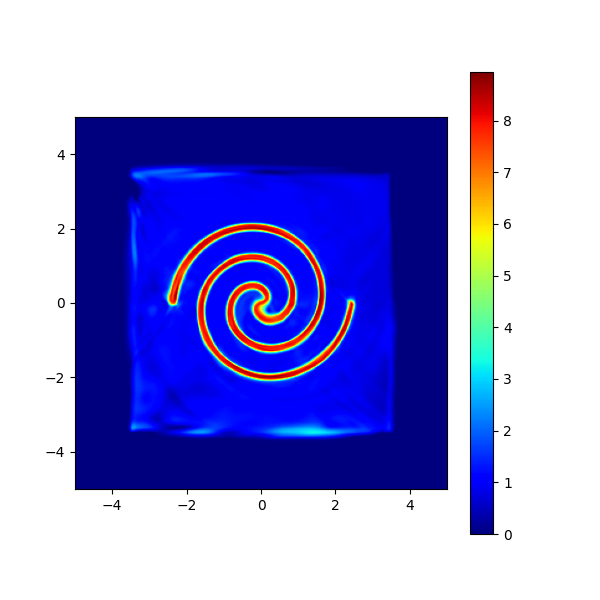}
        \end{subfigure}
        \hfill
        \begin{subfigure}{0.24\linewidth}
            \centering
            \includegraphics[width=\linewidth,trim={3.5cm 4.5cm 5.5cm 4.5cm},clip]{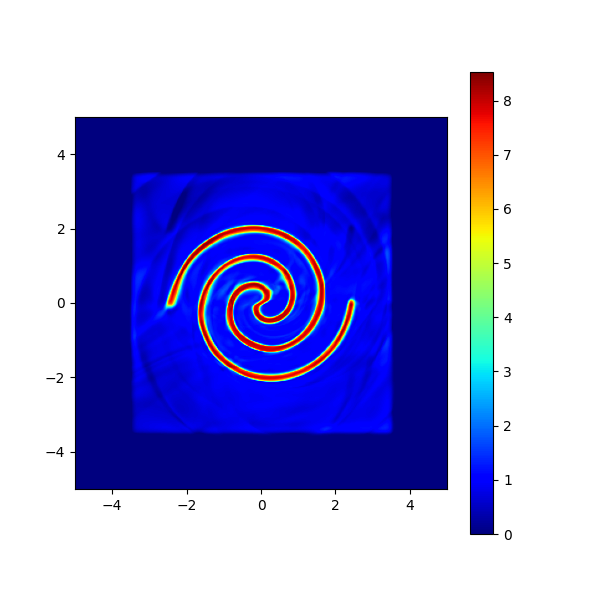}
        \end{subfigure}
        \hfill
        \begin{subfigure}{0.24\linewidth}
            \centering
            \includegraphics[width=\linewidth,trim={3.5cm 4.5cm 5.5cm 4.5cm},clip]{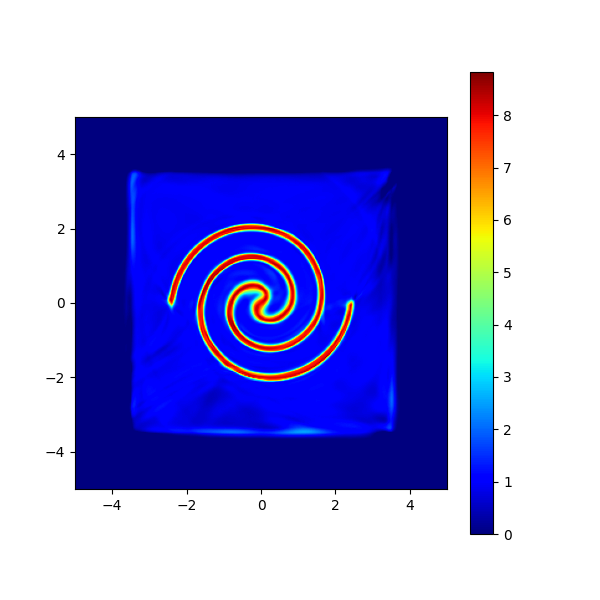}
        \end{subfigure}
        \hfill
        \begin{subfigure}{0.24\linewidth}
            \centering
            \includegraphics[width=\linewidth,trim={3.5cm 4.5cm 5.5cm 4.5cm},clip]{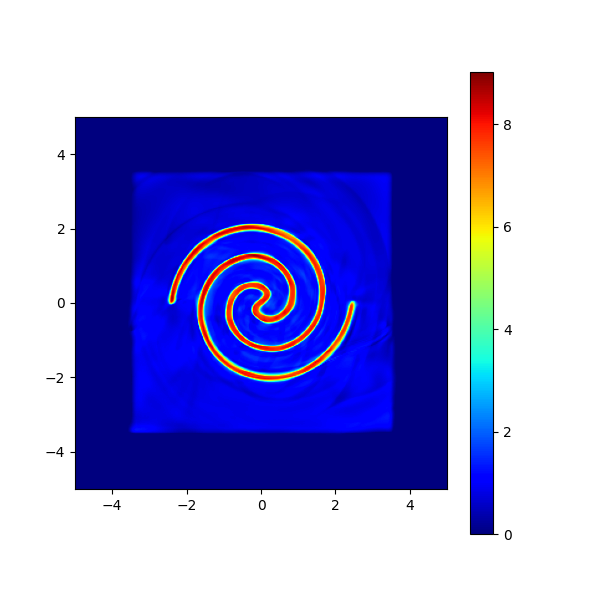}
        \end{subfigure}
        \hfill
       \caption{Generation of two spirals with variations of MLPs for both $\foutstar$ and $\forwardalpha$ from 2x64 to 3x64, then 4x128 and 5x128 (TOP) and 6x64 to 9x64 (BOT). }
    \end{subfigure}

    \caption{Variation of $\foutstar$ and $\forwardalpha$ depths and widths from 2x32 to 9x64, with 4x128 and 5x128 in the second subfigure. Increasing depth and width increases quality until depth 4 then the training seemingly becomes less stable as shown by the 9x64 experiment. Since we keep the learning rate constant at 1e-3, it is likely that the MLPs training becomes unstable.
    The state space is $\Statespace = \mathbb R^2$,  to ensure $\Fout(\Statespace)<+\infty$, we constrain $\foutstar$ to be zero outside $[-4,4]^2$. 
The EGF has 8 transformations which are translations, while the MLPs for $\forwardalpha$ and $\foutstar$ vary from 2x32 (2 hidden layers of width 32) to 9x64, and 4x128 and 5x128. Learning rate is kept at 0.001.
    }
    \label{fig:size_spirals}
\end{figure}


\end{document}